%% file: MFHOO.tex
\documentclass[10pt, twoside]{article}
\usepackage[margin=1in]{geometry}

\usepackage{subfig}

\pdfoutput=1

\usepackage{amsfonts,amsmath,amssymb,amsthm,bm,bbm}

\usepackage{algorithm,algpseudocode}
\usepackage{dsfont,array}

\usepackage{newclude}

\usepackage{appendix}

\usepackage{hyperref}

\usepackage{multirow}
\usepackage{multicol}
\usepackage{xspace}
\usepackage{mathtools}
\usepackage{placeins}
\usepackage{authblk}

\newtheorem{definition}{Definition}
\newtheorem{assumption}{Assumption}
\newtheorem{theorem}{Theorem}
\newtheorem{lemma}{Lemma}
\newtheorem{corollary}{Corollary}

\newtheorem{remark}{Remark}

\newtheorem{condition}{Condition}

 \newcommand{\ignore}[1]{}

\DeclarePairedDelimiter\floor{\lfloor}{\rfloor}
\DeclareMathOperator*{\argmax}{arg\,max}
\newcommand\numberthis{\addtocounter{equation}{1}\tag{\theequation}}

\newcommand{\Xcal}{\mathcal{X}}

\DeclarePairedDelimiterX\Basics[1](){ #1}

\usepackage{centernot}

\newcounter{const-no}

\def\EE{{\mathbb{E}}}\def\PP{{\mathbb{P}}}

\makeatletter
\newcommand{\setword}[2]{%
	\phantomsection
	#1\def\@currentlabel{\unexpanded{#1}}\label{#2}%
}
\makeatother

\date{\today}
\title{Noisy Blackbox Optimization with Multi-Fidelity Queries: A Tree Search Approach}

\author[1]{Rajat Sen}
\author[2]{Kirthevasan Kandasamy}
\author[1]{Sanjay Shakkottai}
\affil[1]{The University of Texas at Austin}
\affil[2]{Carnegie Mellon University}

\usepackage[parfill]{parskip}

\begin{document}

\maketitle

\begin{abstract}
We study the problem of black-box optimization of a noisy function in the presence of low-cost approximations or fidelities, which is motivated by problems like hyper-parameter tuning. In hyper-parameter tuning evaluating the black-box function at a point involves training a learning algorithm on a large data-set at a particular hyper-parameter and evaluating the validation error. Even a single such evaluation can be prohibitively expensive. Therefore, it is beneficial to use low-cost approximations, like training the learning algorithm on a sub-sampled version of the whole data-set. These low-cost approximations/fidelities can however provide a \textit{biased} and \textit{noisy} estimate of the function value. In this work, we incorporate the multi-fidelity setup in the powerful framework of \textit{noisy} black-box optimization through \textit{tree-like} hierarchical partitions. We propose a multi-fidelity bandit based tree-search algorithm for the problem and provide simple regret bounds for our algorithm. Finally, we validate the performance of our algorithm on real and synthetic datasets, where it outperforms several benchmarks. 
\end{abstract}

\input{introductionv3}
\input{rworkv2}
\input{psettingv5-ss}
\input{algorithmv2}

\input{results}

\input{sims}

\input{conc}

\bibliography{mftree}
\bibliographystyle{plain}

\clearpage
\appendix
\input{known_smooth}
\input{relating_to_optimal}
\input{final_guarantees}

\end{document}

%% file: introductionv3.tex
\section{Introduction}
\label{sec:intro}
Several important problems in the fields of physical simulations~\cite{martinez07robotplanning}, industrial engineering~\cite{parkinson06wmap3} and model selection~\cite{snoek12practicalBO} in machine learning can be cast as sequential optimization of a function $f(.)$ over a domain $\mathcal{X}$, with black-box access.
%
%
A black-box optimization algorithm evaluates the function at a set of sequentially chosen points $x_1,...,x_n$ obtaining the function values $f(x_1),...,f(x_n)$ in that order, and outputs a point $\hat{x}(n)$ at the end of the sequence. The performance of such algorithms are commonly evaluated in terms of \textit{simple regret} defined as $S(n) = \sup_{x \in \mathcal{X}}f(x) - f(\hat{x}(n))$.

As the driving example, we consider the problem of tuning hyper-parameters of machine learning algorithms over large datasets. Here, the space of hyper-parameters constitutes the domain $\mathcal{X}$, while the function $f(x)$ represents the validation error after training the machine learning algorithm with the hyper-parameter setting $x \in \Xcal$. Even a single evaluation of the function can be expensive and therefore conventional methods of black-box optimizations are infeasible (e.g. training a deep network over a large data-set could take days). This motivates our setting where access to lower-cost but biased estimates of the function is assumed~\cite{kandasamy2016multi,huang06mfKriging,cutler14mfsim} through a multi-fidelity query framework. In the context of hyper-parameter tuning, a possible low-cost approximation/fidelity can be training and validating the machine learning algorithm over a much smaller sub-sampled version of the data-set. However, the resulting validation error can be a \textit{biased} and \textit{noisy} estimate of the validation error on the whole dataset, and the bias depends on the size of the sub-sampled dataset used.

A well-studied approach for such problems is through bayesian optimization formulations~\cite{kandasamy2017multi,kandasamy2016mfgpucb,huang2006sequential,klein2016fast}. In our work, instead, we approach this problem using \textit{tree-search} methods that have received much recent attention \cite{kocsis2006bandit,kleinberg2008multi,bubeck2011x,munos2011optimistic,valko2013stochastic,sen2018multi}. 
%
%
%
Our study combines an exploration of the state-space using hierarchical binary partitions (a binary tree, with nodes representing subsets of the function domain $\mathcal{X}$ \cite{bubeck2011x,munos2011optimistic,grill2015black}), and queries at these nodes at different fidelities (represented through a continuous parameter over $[0, 1]$), that correspond to lower-cost evaluations of the function, but with both bias and noise.

The \textbf{main contributions} of this paper are as follows:

{\bf \it (i)} We model multiple fidelities in the framework of black-box optimization of a noisy function, with hierarchical partitions in Section~\ref{sec:psetting}. We demonstrate that bandit-based algorithms using hierarchical partitions can be naturally adapted to effectively use a continuous range of low-cost approximations. 

{\bf \it (ii)} We propose MFHOO~(Algorithm~\ref{alg:mfhoo} in Section~\ref{sec:algo}) which is a natural adaptation of HOO~\cite{bubeck2011x} to the multi-fidelity setup. We analyze our algorithm under smoothness assumptions on the function and the partitioning, which are similar to the assumptions adopted in~\cite{grill2015black}. We show that our simple regret guarantees can be much stronger than that of HOO~\cite{bubeck2011x} (which operates only at the highest fidelity) under some natural conditions on the bias and cost functions. Our simple regret guarantees are presented in Theorem~\ref{thm:mfhoo} and Corollary~\ref{cor:geom}. MFHOO however needs the optimal smoothness parameters $(\nu^*,\rho^*)$ as inputs (similar to HOO~\cite{bubeck2011x}, see Section~\ref{sec:psetting} for details). In Section~\ref{sec:algo}, we propose a second algorithm MFPOO (Algorithm~\ref{algo:robust_algo}) that can achieve a simple regret bound close to that of MFHOO, even without access to the parameters $(\nu^*,\rho^*)$. MFPOO is inspired by the recent techniques proposed in~\cite{grill2015black}. Theorem~\ref{thm:mfpoo} provides simple regret guarantees for MFPOO. 

{\bf \it (ii)} Finally, in Section~\ref{sec:sims} we empirically compare the performance of our algorithm on real and synthetic data, against state of the art algorithms for black-box optimization. First, we perform simulated experiments on well-known benchmark functions. Then, we compare our algorithm to others in the context of  hyper-parameter tuning for regression and classification tasks. 

%% file: rworkv2.tex
\section{Related Work}
\label{sec:rwork}
This work builds on a rich literature on black-box optimization using hierarchical partitions~\cite{munos2011optimistic,valko2013stochastic,kleinberg2008multi,bubeck2011x}, which in turn build on optimistic algorithms in a bandit setting \cite{auer2002finite}.
Our work is closely related to~\cite{bubeck2011x} which introduces HOO, a tree-search based algorithm for noisy black-box optimization. The regret guarantees in~\cite{bubeck2011x} are provided under some local Lipschitz assumption on the function with respect to a semi-metric and also some assumptions on the diameter of the nodes as a function of height in the tree. Later these assumptions were combined into a single combinatorial assumption in~\cite{grill2015black}. We follow similar assumptions. In a recent related work~\cite{sen2018multi}, the multi-fidelity tree-search problem has been studied in a setting where the evaluations of the black-box function are \textit{not noisy}. In contrast we study the problem of multi-fidelity black box optimization in the presence of noise, which makes the setting much more challenging. In particular in the deterministic setting, one can simply descend through the tree without back-tracking. However in our setting, back-tracking in a tree occurs as time progresses because additional samples improves estimates at various nodes in the tree, thus introducing bandit explore-exploit trade-offs, and hence can change their relative ordering over time.
This distinction leads to a considerably different algorithm and analysis in our paper as compared to~\cite{sen2018multi}. 

Multi-fidelity optimization is also relevant in several application~\cite{forrester07cokriging,lam2015multifidelity,huang06mfKriging,poloczek2016multi,klein2016fast}, but theoretical guarantees are generally lacking in these studies. Recently, the multi-fidelity setting has been theoretically studied in online problems~\cite{zhang15weakAndStrong,agarwal2011oracle,sabharwal2015selecting,kandasamy2016multi}. In some recent works~\cite{kandasamy2017multi,kandasamy2016gaussian,kandasamy2016mfgpucb}, UCB like algorithms with Bayesian Gaussian process assumptions on $f$ have been analyzed in a multi-fidelity black-box optimization setting.

Also relevant to this work are the bandit based techniques for hyper-parameter optimization such as~\cite{li2016hyperband,jamieson2016non}. However, these methods rely on iterative loss sequences and therefore are not directly applicable to the general multi-fidelity setup.

%% file: psettingv5-ss.tex
\section{Problem Setting}
\label{sec:psetting}
The problem setting in this work is that of optimizing a function $f : \mathcal{X} \rightarrow \mathbb{R}$ with \textit{noisy and biased}  black-box access. Given a finite cost budget $\Lambda,$ and access to fidelity-dependent queries (with higher cost for higher fidelity and lower bias), we need to determine $\hat{x} \in  \mathcal{X}$ such that $|\sup_{x\in\Xcal} f(x) - f(\hat{x})|$ is small.
%
%
A similar problem setting has been considered in a recent work~\cite{sen2018multi}, however there the function evaluations at different fidelities are \textit{not noisy}, that is the cheap approximations only add \textit{bias} to the function values. In this work, we consider a setting where a function evaluation at a cheaper fidelity incurs both bias and sub-Gaussian noise around the function value. This makes the problem setting different and significantly more challenging. 

\subsection{Details about Function Evaluations}

We build on the notation in \cite{sen2018multi}, but modify it to permit noisy evaluations. The fidelity is modeled as a continuous parameter $z \in \mathcal{Z} \triangleq [0,1].$ An evaluation of a function consists of two inputs $(x,z)$ corresponding to the query point $x \in \mathcal{X}$ and fidelity $z \in \mathcal{Z},$ and results in a random variable $Y$ as the output, with $Y = f_z(x) + \epsilon.$
Here, $\epsilon$ is a sub-Gaussian noise random variable~\cite{buldygin1980sub} with parameter $\sigma^2$ i.e $\EE[\epsilon] = 0$ and $\EE[ \exp(s\epsilon)] \leq \exp (s^2\sigma^2/2)$ all $s \geq 0$. Such random variables will be denoted as $\mathrm{subG}(\sigma)$. 

In our model, the mean of the query, $ f_z(x),$ is biased and progressively smaller bias can be obtained but with higher costs. This is formalized through the monotone decreasing \textit{bias function} (assumed to be known in our analysis\footnote{Note that the knowledge of the bias function is only required for our theoretical analysis, similar to prior works~\cite{kandasamy2017multi,kandasamy2016gaussian}. In practice, we can assume a simple parametric form of the bias function which can be updated online. We provide more details in Section~\ref{sec:sims}.}) $\zeta: \mathcal{Z} \rightarrow \mathbb{R}^+$ and a monotone increasing \textit{cost function} $\lambda: \mathcal{Z} \rightarrow \mathbb{R}^+,$ with  $\lambda(z) \geq 1.$ A query at $(x, z)$ results in an output mean with $\lvert f_z(x) - f(x) \rvert \leq \zeta(z)$ and at cost $\lambda(z).$ We assume that at the highest fidelity ($z = 1$), there is no bias, i.e.,  $f_1(x) = f(x).$ Finally, we assume that the global optimum is unique, i.e. $\exists x^* = \argmax_{x \in \mathcal{X}} f(x).$ 

Our problem setting is a natural model for the driving example of hyper-parameter tuning. For instance, the domain $\mathcal{X}$ represents the range of hyper-parameters. In the case of deep networks this may include kernel sizes and number of channels in different layers (rounded off to the nearest integers), learning rate of optimizers, dropout level and weight decay levels etc. While working with large datasets, training and validating on sub-sampled versions of the datasets serve as cheap approximations which can be modeled by our continuous range of fidelities $\mathcal{Z}$. For instance on Cifar10~\cite{krizhevsky2009learning}, using the whole training set with $N_s > 50k$ samples corresponds to $z=1$ while $z=0$ can correspond to using $n_s \sim 1000$ samples. The range [1k,50k] can be linearly mapped to $[0,1]$ (rounding off to the nearest interger number of samples). It is also clear that function evaluations obtained at the cheaper fidelities are both biased (validation accuracy is lower while using lesser number of samples) and noisy (randomness in sub-sampling and SGD). This falls under the purview of our evaluation model. The cost function simply corresponds to the time required to train a model over a sub-sampled version of the data-set and it increases with the number of samples (fidelity).

Putting all the pieces together, the problem setting is a sequential process as follows: Given a cost budget $\Lambda$ at each time-slot $t = 1,2,..$, $(i)$ Select a point $X_t$ and evaluate it at a fidelity $Z_t$, $(ii)$ Observe a noisy feedback $Y_t$, $(iii)$ Incur a cost $\lambda(Z_t).$ This process in continued till the cost budget is exhausted. The goal is to find a point $x \in \mathcal{X}$ such that $f(x)$ is as close as possible to $f^* \triangleq f(x^*)$. We assume that $x^*$ is a unique point in $\mathcal{X}$ where the supremum is achieved. The performance of a policy/optimization algorithm shall be characterized by \textit{simple regret}. 

\subsection{Definition of Regret}
 Let $X_1,X_2,....$ be the points queried by an algorithm/policy at fidelities $Z_1,Z_2,...$ respectively. Let $\hat{X}_{\Lambda}$ be the point returned by the  policy after $N(\Lambda)$ queries, where $N(\Lambda)$ is a random quantity such that $N(\Lambda) = \max\{n \geq 1: \sum_{i=1}^{n} \lambda(Z_i) \leq \Lambda\}$ and $\Lambda$ is the total cost budget. Then the \textit{simple regret} is given by $S(\Lambda) = f^* - \EE[f(\hat{X}_{\Lambda})]$. Here, the expectation is over the randomness in the observations and the policy. Note that the regret is measured only at the highest fidelity that is $z = 1$, as we are interested in optimizing the original function at $z=1$. A similar definition of regret has been used before in the multi-fidelity literature~\cite{kandasamy2016multi,kandasamy2017multi}. 
In the course of our analysis, we are also interested in the \textit{cumulative regret} as an intermediate quantity. Given that a policy performs $n$ evaluations or queries, the cumulative regret is given by $R_n = \EE \left[ \sum_{t = 1}^{n} [f^* - f(X_t)]\right].$
  
The problem of black-box optimization cannot be solved efficiently without assuming any structure or regularity assumptions on the function being evaluated. In the next sub-section, we assume access to a hierarchical partitioning of the domain $\mathcal{X}$ and impose some regularity assumption jointly on the function and the hierarchical partitions. Similar assumptions have been used in the theoretical analysis of several prior works~\cite{munos2011optimistic,grill2015black,bubeck2011x,kleinberg2008multi}, that work with hierarchical partitioning of the domain. 

\subsection{Tree-like Partitions and Structural Assumptions}
\label{sec:smooth_assume}
We first define the tree-like hierarchical partitions of the domain $\mathcal{X}$ and then provide some technical assumptions that impose regularity conditions jointly on the function and the partitions.   

{\bf Hierarchical Partitions: } We assume that the domain $\mathcal{X}$ is partitioned hierarchically according to an infinite binary tree. We denote this partitioning as $\mathcal{P} = \{(h,i)\}$, where $h$ is a depth parameter and $i$ is an index. For any depth $h \geq 0$, the cells $\{(h,i)\}_{1 \leq i \leq I_h}$ denote a partitioning of the space $\mathcal{X}$. Note that we use the notation $(h,i)$ to denote both the index of a cell/node and the part of the domain it represents. For example, $x \in (h,i)$ refers to a point $x$ within the part of the domain represented by the cell indexed at $(h,i)$. At depth $0$ there is a single cell $(0,1) = \mathcal{X}$.  A cell $(h,i)$ can be split into two child nodes at depth $h+1$, which are indexed as $(h+1,2i-1)$ and $(h+1,2i)$. A cell $(h,i)$ is said to be \textit{queried}, when a fixed representative point $x_{h,i} \in (h,i)$ (ideally \textit{centrally} located) is evaluated at any fidelity. Let $\mathcal{C}(h,i)$ denote all the descendant cells of $(h,i)$ in the infinite tree. The unique cell at height $h$ that contains the optima $x^*$ is indexed as $(h,i^*)$. For all sub-optimal cells $(h,i)$ (those that do not contain $x^*$) the sub-optimality gap is denoted by $\Delta_{h,i} = f^* - \sup_{x \in (h,i)}f(x)$. In other words, a sub-optimal cell $(h,i)$ is $\Delta_{h,i}$-optimal. Let $\mathcal{X}_{\epsilon} = \{x \in \mathcal{X}: f(x) \geq f^* - \epsilon\}$ for all $\epsilon > 0$.


For instance, consider the domain $\mathcal{X} = [0,1]\times[0,1] \subset \mathbb{R}^2$. In this example, we only consider cells that intervals $\mathcal{C} = \{x \triangleq [x_1,x_2] \in \mathcal{X}: b_{1,l} < x_1 \leq b_{1,u}, b_{2,l} < x_2 \leq b_{2,u} \} \triangleq [[b_{1,l},b_{1,u}],[b_{2,l},b_{2,u}]]$. In this case, the hierarchical partitioning has a root $(0,1) = [[0,1],[0,1]]$ (the entire domain). At $h=1$, the two children of the root resulting form a split along $x_1$ would be $(1,1) = [[0,0.5],[0,1]] $ and $(1,2) = [[0.5,1],[0,1]] $. The cell $(1,1)$ can be split into two children along the $x_2$ coordinate to result in $[[0.5,1],[0,0.5]]$ and $[[0.5,1],[0.5,1]],$ with the coordinate-wise midpoint of the cell used as the representative point $x_{h,i}$.
In our example, $x_{1,1}$ can be chosen as the point $[0.25,0.5]$.  

We impose the following joint assumptions on the hierarchical partition $\mathcal{P}$ and the black-box function $f$, similar to the recent works~\cite{grill2015black,sen2018multi}.  




\begin{assumption}
	\label{as:partition}
	There exists $\nu$ and $\rho \in (0,1)$ such that for all cells $(h,i)$ such that $\Delta_{h,i} \leq c\nu\rho^h$ (for a constant $c \geq 0$) we have that,
$f(x) \geq f^* - \max\{2c,c+1\} \nu \rho^h, \text{ for all } x \in (h,i).$

\end{assumption}


Finally, the following definition of \textit{near-optimality-dimension} with parameters $(\nu,\rho)$ is borrowed from~\cite{grill2015black}.

\begin{definition}
	\label{def:dimension}
	The near-optimality dimension of $f$ with respect to parameters $(\nu,\rho)$ is given by,
	\begin{align*}
	d(\nu,\rho) &\triangleq  \inf \left\{ d' \in \mathbb{R}^+ : \exists C(\nu,\rho), \forall h \geq 0, \right. \\
	&\left. ~ \mathcal{N}_{h}(2\nu \rho ^h) \leq  C(\nu,\rho) \rho ^{-d'h} \right\} \numberthis
	\end{align*}
	where $\mathcal{N}_{h}(\epsilon)$ is the number of cells $(h,i)$ such that $\sup_{x \in (h,i)} f(x) \geq f(x^*) - \epsilon$. 
\end{definition} 
We denote the parameters associated with the minimum near optimality dimension $d(\nu,\rho)$ to be  $(\nu^*,\rho^*).$ 
The optimal near-optimality dimension $d(\nu^*,\rho^*)$ controls the hardness of optimizing the function, given access to the particular hierarchical partition.

Our assumptions are closely related to the ones in the seminal paper~\cite{bubeck2011x}. Bubeck et al.~\cite{bubeck2011x} consider a similar noisy tree-search based black-box optimization problem. In their work, it was assumed that there is a dissimilarity metric $\ell(x,y)$ over the domain and the function satisfies a weak-Lipschitz condition around the optima with respect to the dissimilarity. These assumptions have been progressively refined \cite{munos2011optimistic,valko2013stochastic}, with \cite{grill2015black} providing a succinct assumption using the framework of hierarchical partitions.
As in \cite{sen2018multi}, we adopt this assumption in our paper. Assumption~\ref{as:partition} is a slightly stronger version of Assumption 1 in~\cite{grill2015black} i.e., in~\cite{grill2015black} it has been assumed that Assumption~\ref{as:partition} is satisfied with only $c = 0$. It has been recently observed~\cite{shang2017adaptive} that it is highly non-trivial to prove the regret guarantees of HOO~\cite{bubeck2011x} under the assumptions in~\cite{grill2015black} and this stronger version may be indeed necessary. Assumption~\ref{as:partition} is akin to ensuring that the conditions of Lemma 3 in~\cite{bubeck2011x}  are satisfied. 

%% file: algorithmv2.tex
\section{Algorithms}
\label{sec:algo}
We first propose MFHOO (Multi-Fidelity Hierarchical Optimistic Optimization) which is a noisy tree-search based multi-fidelity black-box optimization policy that requires the optimal smoothness parameters as input. Then we propose another algorithm MFPOO (Multi-Fidelity Parallel Optimistic Optimization) that can recover regret guarantees similar to that of MFHOO without the exact knowledge of smoothness parameters. 

{\bf When $(\nu^*,\rho^*)$ are known: }  Our first algorithm MFHOO is inspired by the HOO strategy in~\cite{bubeck2011x}.  We essentially show that the tree-search based technique in~\cite{bubeck2011x} can be naturally adapted to a multi-fidelity setting, with some modifications. In certain settings, our algorithm can achieve a much stronger simple regret scaling when compared to HOO which queries only at fidelity $z = 1$. The detailed pseudo-code of the algorithm is provided as Algorithm~\ref{alg:mfhoo}. We first establish some notation specific to our algorithm.  

For any black-box optimization policy, let $X_t$ be the random variable denoting the point queried at time $t$ which is part of the cell $(H_t,I_t)$, while $Z_t$ is the fidelity at which the query is made. Let $Y_t$ be the observation at the corresponding time-step such that $Y_t = f_{Z_t}(X_t) + \epsilon_t$, where $\epsilon_t \sim \mathrm{subG}(\sigma)$. Let $T_{h,i}(t)$ be the number of times nodes in $\mathcal{C}(h,i)$ have been queried i.e,
$T_{h,i}(t) = \sum_{s = 1}^{t} \mathds{1}\{(H_s,I_s) \in \mathcal{C}(h,i) \}$. Let $\mathcal{T}_t$ denote the finite subtree visited by the algorithm at the end of round $t$. The tree is initialized at $\mathcal{T}_0 = \{(0,1)\}$. Now we are at a position to introduce Algorithm~\ref{alg:mfhoo}.

\begin{algorithm}[!ht]                      
	\caption{MFHOO: Multi-Fidelity Hierarchical Optimistic Optimization}          
	\label{alg:mfhoo}                           
	\begin{algorithmic}[1]                    
		\State {\bf Inputs} - Cost Budget: $\Lambda$, Sub-Gaussian Parameter: $\sigma$, Partitioning Structure: $(h,i)$, Bias function: $\zeta(.)$, Cost function: $ \lambda(.)$, Smoothness Parameters: $(\nu,\rho)$. 
		\State {\bf Initialization} - $\mathcal{T} =\{(0,1)\}$, $B_{1,2} = B_{2,2} = \infty$, $C = 0$ and $n=0$.
		\While {$C \leq \Lambda$}
		\State $(h,i) \leftarrow (0,1)$.
		\State $P \leftarrow \{(h,i)\}$.
		\While {$(h,i) \in \mathcal{T}$}
		\State Update $(h,i)$ to $(h+1,2i-1)$ if $B_{h,2i-1} > B_{h,2i}$ or to $(h+1,2i)$ if $B_{h,2i-1} < B_{h,2i}$. 
		\State Ties are broken at random. $P \leftarrow P \cup \{(h,i)\}$. 
		\EndWhile
		\State $(H,I) \leftarrow (h,i)$. Query $x_{H,I}$ at fidelity $z_H$ and receive value $Y$. $\mathcal{T} \leftarrow \mathcal{T} \cup \{(H,I)\}$. 
		\State Let $n = n+1$ and let $x_n \triangleq x_{H,I}$. Update $C = C + \lambda(z_h)$. 
		\For {all $(h,i) \in P$}
		\State $T_{h,i} \leftarrow T_{h,i} + 1$
		\State $\hat{\mu}_{h,i} \leftarrow (1 - 1/T_{h,i})\hat{\mu}_{h,i} + Y/T_{h,i}.$ 
		\EndFor
		\For {all $(h,i) \in P$}
		\State $U_{h,i} \leftarrow \hat{\mu}_{h,i} + \sqrt{2 \sigma^2 \log n / T_{h,i}} + \nu\rho^h + \zeta(z_h).$ 
		\EndFor
		\State $B_{H+1,2I-1} = B_{H+1,2I} = \infty$
		\State Starting from the leaves down to the root maintain: $B_{h,i} \leftarrow \min \{U_{h,i}, \max\{B_{h+1,2i-1},B_{h+1,2i}\}\}$. 
		\EndWhile
		\State Return a point among $x_1,x_2...,x_n$ chosen uniformly at random. 
	\end{algorithmic}
\end{algorithm}

The notable difference from HOO~\cite{bubeck2011x} is that all queries at height $h$ are performed at a fidelity $z_h$ such that $\zeta(z_h) = \nu\rho^h$. The intuition is that in 'near optimal' cells at height $h$, the function values of all points inside a cell are at most $O(\nu\rho^h)$ apart from each other. Therefore, if $x^*$ belongs to a cell $(h,i^*)$ at height $h$, then all points in that cell are $O(\nu\rho^h)$ optimal. Thus in the absence of noise, ideally beyond this point we would only like to expand nodes/cells that are at least $O(\nu\rho^h)$ optimal, which is only possible if the error due to the fidelities is $O(\nu\rho^h)$.

\begin{remark}
	\label{remark:practice}
	Note that in the pseudo-code of Algorithm~\ref{alg:mfhoo}, the final point returned is randomly chosen from the points evaluated in the course of the algorithm. This is sufficient in theoretically bounding the simple regret as in Theorem~\ref{thm:mfhoo}. However, in practice several optimizations can be performed to return the most promising point among the ones evaluated. In our implementation we return a point $\hat{x}_{\Lambda} \in \{x_1,....,x_n \}$ such that $\hat{x}_{\Lambda} = \argmax_{x_i} f_{z_i}(x_i) - \zeta(z_i) + \epsilon_i$. Note that $ f_{z_i}(x_i) - \zeta(z_i)$ is a lower bound on the value $f(x_i)$. We return a point that approximately maximizes this lower-bound.  
\end{remark}

{\bf When $(\nu^*,\rho^*)$ are not known: } Grill. et al.~\cite{grill2015black} have recently developed a technique for searching for the optimal smoothness parameters for HOO~\cite{bubeck2011x}. The technique can be extended to our algorithm MFHOO in the multi-fidelity setup. This leads us to our second algorithm MFPOO (Algorithm~\ref{algo:robust_algo}). 

\begin{algorithm}                      
	\caption{MFPOO: Multi-Fidelity Parallel Optimistic Optimization}          
	\label{algo:robust_algo}                           
	\begin{algorithmic}[1]   
		\State {\bf Arguments: } $(\nu_{max},\rho_{max})$, $\zeta(z)$, $\lambda(z)$, $\Lambda$, $\sigma$                 
		\State Let $N = (1/2)D_{max} \log (\Lambda/\log(\Lambda))$ where $D_{max} = \log 2 / \log (1/\rho_{max})$
		\For {$i = 1$ to $N$}
		\State Spawn MFHOO with parameters $(\nu_{max},\rho_i = \rho_{max}^{N/(N-i-1)})$ with budget $(\Lambda - N\lambda(1))/N$
		\EndFor 
		\State Let $\hat{x}_{\Lambda,i}$ be the point returned by the $i^{th}$ MFHOO instance for $i \in \{0,..,N-1 \}$. Evaluate all $\{\hat{x}_{\Lambda,i}\}_{i}$ at $z = 1$. Return the point $\hat{x}_{\Lambda} = \hat{x}_{\Lambda,i^*}$ where $i^* = \argmax_{i} f(x_{\Lambda,i}) + \epsilon_i$. 
	\end{algorithmic}
\end{algorithm} 

The key idea of the algorithm is to spawn several MFHOO instances with different smoothness parameters $\rho_1,...,\rho_N$. The sequence $\rho_1,...,\rho_N$ is chosen carefully according to the strategy introduced in~\cite{grill2015black}. The budget is uniformly allocated in between all the MFHOO instances spawned. The $i$-th MFHOO instance is spawned with the parameters $(\nu_{max},\rho_i = \rho_{max}^{N/(N-i-1)})$. It is only required that $\rho_{max} \geq \rho^*$ and $\nu_{max} \geq \nu^*$. In Theorem~\ref{thm:mfpoo} we show that at least one of the MFHOO instances spawned by MFPOO has a simple regret guarantee of MFHOO run at the optimal parameters $(\nu^*,\rho^*)$ but with a budget $(\Lambda - N\lambda(1))/N$. We provide more details and intution about this algorithm in Appendix~\ref{sec:mfpoo}.

%% file: results.tex
\section{Theoretical Results}
\label{sec:results}
In this section we provide our main theoretical results: Simple regret bounds for MFHOO (Algorithm~\ref{alg:mfhoo}) and MFPOO (Algorithm~\ref{algo:robust_algo}). First we present Theorem~\ref{thm:mfhoo}, which provides a simple regret bound for Algorithm~\ref{alg:mfhoo}. 

\begin{theorem}
	\label{thm:mfhoo}
	If Algorithm~\ref{alg:mfhoo} is run with parameters $(\nu,\rho)$ that satisfy Assumption~\ref{as:partition} and given a total cost budget $\Lambda$, then the simple regret is bounded as follows,
	\begin{align*}
	S(\Lambda) &= \mathcal{O} \left( C(\nu,\rho)^{\frac{1}{d(\nu,\rho)+2}} n(\Lambda)^{-\frac{1}{d(\nu,\rho)+2}} \times \right. 
	 \left.(\log n(\Lambda))^{1/(d(\nu,\rho)+2)} \right),
	\end{align*} 
	where $n(\Lambda) = \max \{n : \sum_{h=1}^{n} \lambda(z_h) \leq \Lambda \}$. Here, $z_h = \zeta^{-1}(\nu\rho^h)$.  
\end{theorem} 

{\bf Comparison with HOO~\cite{bubeck2011x}:} The simple regret bound that is attained by HOO~\cite{bubeck2011x} (operating at the highest fidelity) given the same cost budget $\Lambda$ is $S'(\Lambda) = \mathcal{O} ( (\Lambda/\lambda(1))^{-1/(d(\nu,\rho)+2)} (\log (\Lambda/\lambda(1)))^{1/(d(\nu,\rho)+2)} ).$ It is easy to verify that $S(\Lambda) < S'(\Lambda)$, as $\lambda(z_h) \leq \lambda(1)$ for all $z_h$. In many real-world situations like hyper-parameter tuning the regret of MFHOO can be much less as compared to HOO operating at the highest fidelity. In fact the real gain in MFHOO is observed in situations where evaluating at the highest fidelity is extremely expensive and $\Lambda$ is of the order of $\lambda(1)$. We will now provide a corollary that highlights this, which is motivated by the following \textit{illustrative example}.  The setting below and analogous corollaries for the {\em noiseless case} is available in \cite{sen2018multi}.

{\bf Illustrative Example: } Let us consider our hyper-parameter tuning example again, however let us use the fidelity range to model the number of iterations of an iterative learning algorithm. For concreteness, we will assume that the learning iterations are gradient descent steps on a smooth strongly convex objective. Let $z=1$ represent training to completion which might take $N$ iterations or descent steps. Cheaper fidelities correspond to training for fewer iterations and validating, for instance $z_n < 1$ corresponds to training till $n < N$ iterations and is $O(n/N)$. The cost is linearly proportional to the fidelity, while the error of gradient descent at fidelity $z_n$ is $O(r^n)$ for some $r \in (0,1)$. Thus, if  $\zeta(z_n)$ scales as $\nu_*\rho_*^h$, then $n$ scales as $O(h0).$  It then follows that $\lambda(z_n) = O(h).$
%
%
It should also be noted that in the context of optimizing deep networks, where training till completion can take many hours, the total cost budget $\Lambda$ is usually a small multiple of $\lambda(1)$ (evaluation cost at the highest fidelity). This motivates the following condition and the corollary that follows.

\begin{condition}
\label{cond:geom} 	$\zeta(.)$ and $\lambda(.)$ are such that $\lambda(z^*_h) \leq \min\{\beta h,\lambda(1)\}$ for some constant $\beta > 0$. Here, $z^*_h = \zeta^{-1}(\nu_*\rho_*^h)$. Further, we assume that $\Lambda \leq \lambda(1)^{1+\epsilon}$ for some $\epsilon \in (0,1)$. Here, $\beta$ is a universal constant which is much less than $\lambda(1)$. 
\end{condition}

Under Condition~\ref{cond:geom}, we get the following corollary of Theorem~\ref{thm:mfhoo}. 
\begin{corollary}
\label{cor:geom}
If Algorithm~\ref{alg:mfhoo} is run with parameters $(\nu,\rho)$ that satisfy Assumption~\ref{as:partition} such that $\nu > \nu^*$ and $\rho > \rho^*$ with a total cost budget $\Lambda$, then the simple regret is bounded under Condition~\ref{cond:geom} as,
\begin{align*}
S(\Lambda) &= \mathcal{O} \left( C(\nu,\rho)^{\frac{1}{d(\nu,\rho)+2}} n_g(\Lambda)^{-\frac{1}{d(\nu,\rho)+2}} \times \right. 
\left.(\log n_g(\Lambda))^{1/(d(\nu,\rho)+2)} \right),
\end{align*}
where $n_g(\Lambda) \geq \sqrt{2 (\Lambda - \lambda(1))/\beta}.$ 
\end{corollary}

Thus, Corollary~\ref{cor:geom} implies that under Condition~\ref{cond:geom} the simple regret of MFHOO (Algorithm~\ref{alg:mfhoo}) scales as $S(\Lambda) = \mathcal{O}(( \log \Lambda/\sqrt{\Lambda} )^{1/(d(\nu,\rho)+2)})$. On the other hand, HOO~\cite{bubeck2011x} would only be able to evaluate $\Lambda/\lambda(1)$ points. Thus, the simple regret of  HOO would scale as $S'(\Lambda) = \mathcal{O}\left(\left( \log \Lambda/\Lambda^{\epsilon/(1+\epsilon)} \right)^{1/(d(\nu,\rho)+2)}\right)$, as $\Lambda \leq \lambda(1)^{1 + \epsilon}$.  Thus, in this setting $S(\Lambda)$ can be order-wise less than $S(\Lambda')$, as $\epsilon < 1$.  
%

Our next result in Theorem~\ref{thm:mfpoo} shows that at least one of the MFHOO instances spawned by Algorithm~\ref{algo:robust_algo} has a simple regret close to that of an MFHOO run with the parameters $(\nu^*,\rho^*)$. Thus, MFPOO (Algorithm~\ref{algo:robust_algo}) can recover the performance of MFHOO run with the optimal parameters when supplied with just an upper bound on $\nu^*$ and $\rho^*$ respectively. 

\begin{theorem}
\label{thm:mfpoo}
If Algorithm~\ref{algo:robust_algo} is run with parameters $(\nu_{max} (\geq \nu^*),\rho_{max} (\geq \rho^*))$ and given a total cost budget $\Lambda$, then the simple regret of at least one of the MFHOO instances spawned by  Algorithm~\ref{algo:robust_algo} is bounded as follows,
\begin{align*}
S(\Lambda) &= \mathcal{O} \left((\nu_{max}/\nu^*)^{D_{max}} C(\nu^*,\rho^*)^{\frac{1}{d(\nu^*,\rho^*)+2}} \times \right. 
\left.  \left( \frac{\log n(\Lambda/\log \Lambda)}{n(\Lambda/\log \Lambda)}\right)^{\frac{1}{2 + d(\nu^*,\rho^*)}} \right). \numberthis 
\end{align*}
\end{theorem} 

The simple regret bound in Theorem~\ref{thm:mfpoo} should be compared to that of Theorem~\ref{alg:mfhoo} when Algorithm~\ref{alg:mfhoo} is run with the best parameters $(\nu^*,\rho^*)$. The expression is Theorem~\ref{thm:mfpoo} is order-wise same as  the simple regret achieved by MFHOO run at the optimal parameters $(\nu^*,\rho^*)$ but with a budget of $\Lambda/\log \Lambda$. This is a minor loss in terms of simple regret and is achieved without exact knowledge of the optimal parameters. For instance, under Condition~\ref{cond:geom}, the simple regret of MFPOO is only a factor of $O((\log \Lambda)^{1/(2+d(\nu^*,\rho^*))})$ away from that of MFHOO run with parameters $(\nu^*,\rho^*)$. Note that there are differences between the style of results in~\cite{grill2015black} and Theorem~\ref{thm:mfpoo} ( more details in Appendix~\ref{sec:mfpoo}). 


%% file: sims.tex
\section{Empirical Results}
\label{sec:sims}
\begin{figure*}[!ht]
	\centering
	\subfloat[][]{\includegraphics[width = 0.32\linewidth]{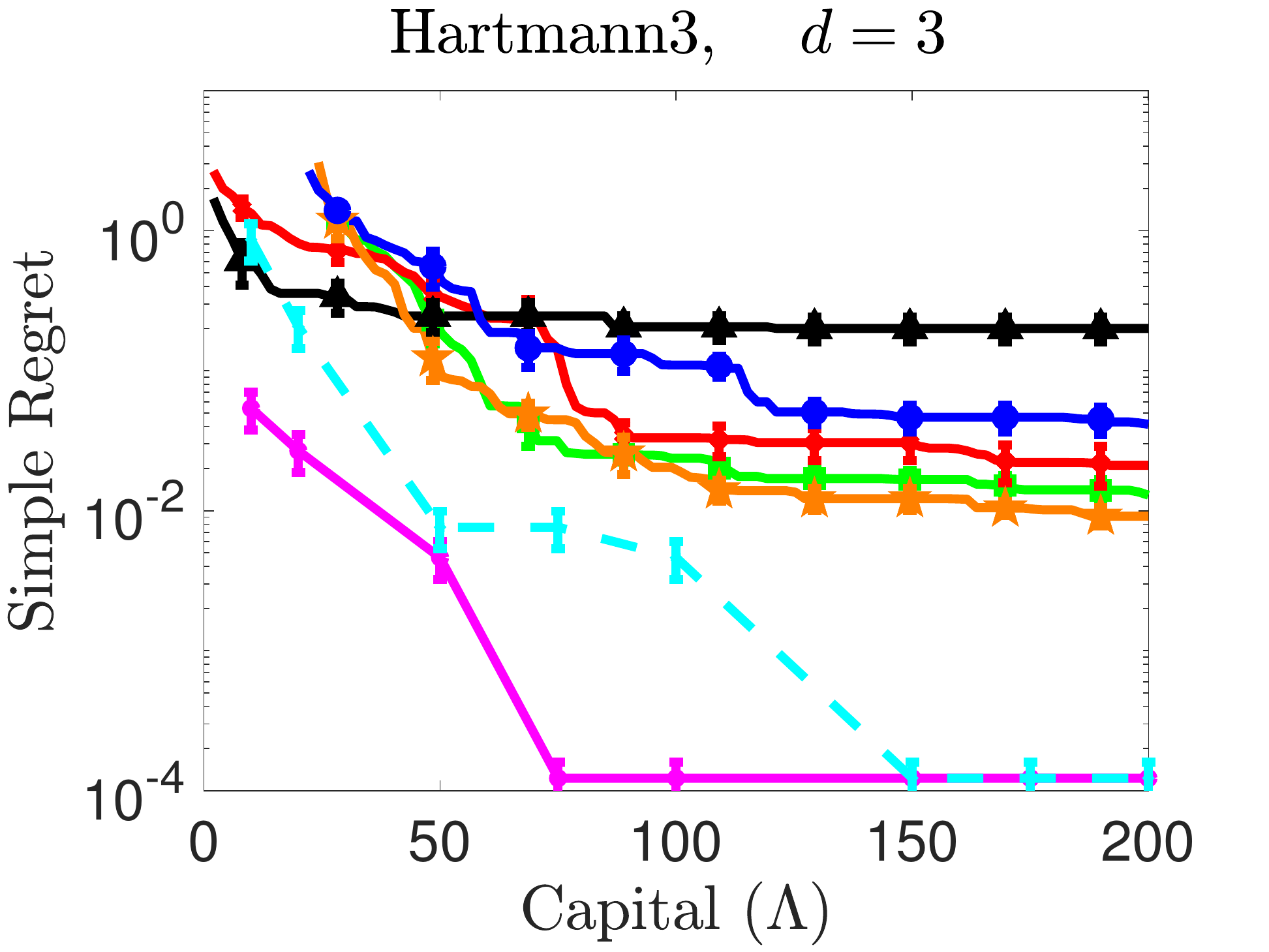}\label{fig:h3}} \hfill
	\subfloat[][]{\includegraphics[width = 0.32\linewidth]{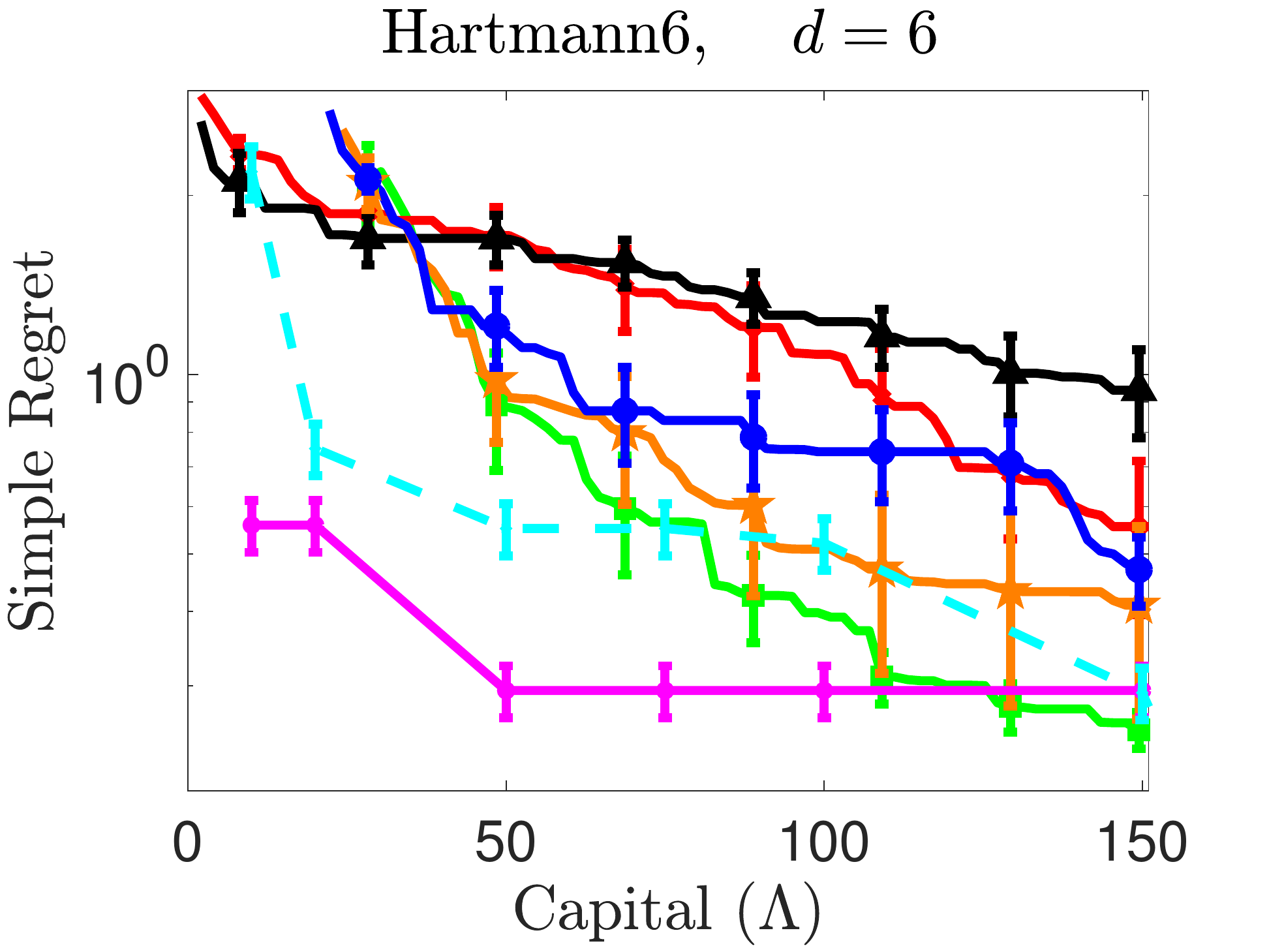}\label{fig:h6}} \hfill
	\subfloat[][]{\includegraphics[width = 0.32\linewidth]{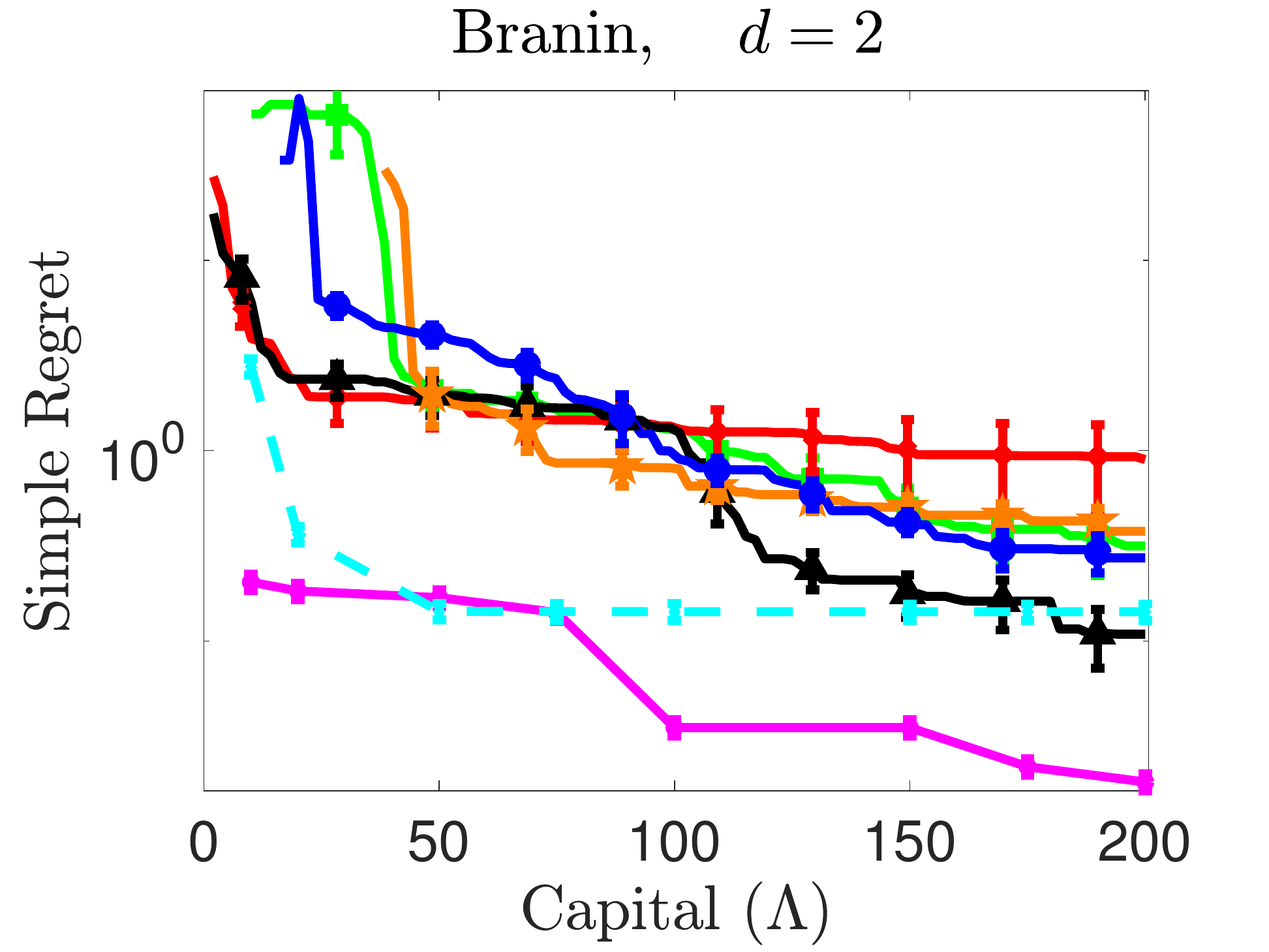}\label{fig:b}} \\
	\subfloat[][]{\includegraphics[width = 0.32\linewidth]{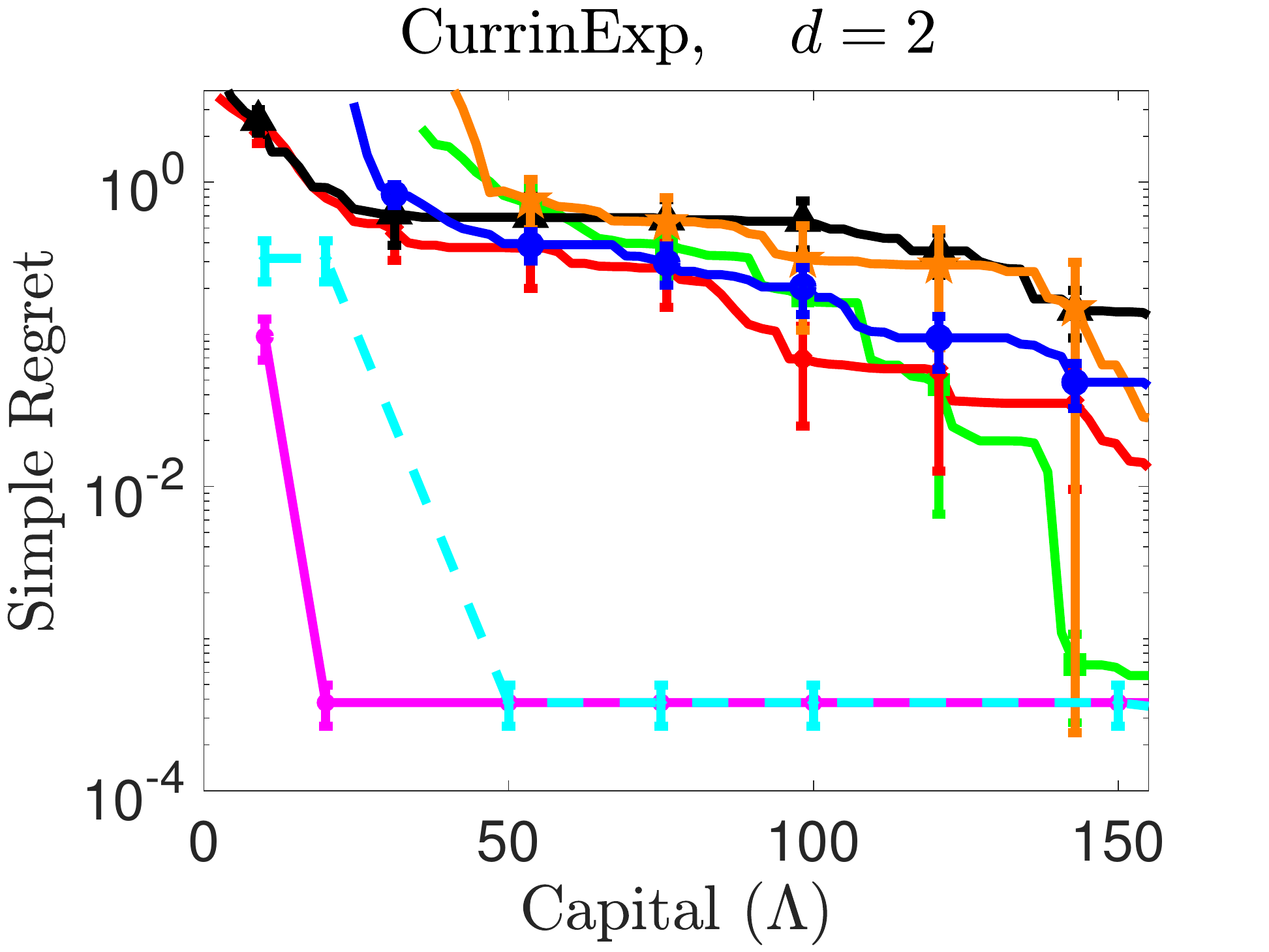}\label{fig:c}} \hfill
	\subfloat[][]{\includegraphics[width = 0.32\linewidth]{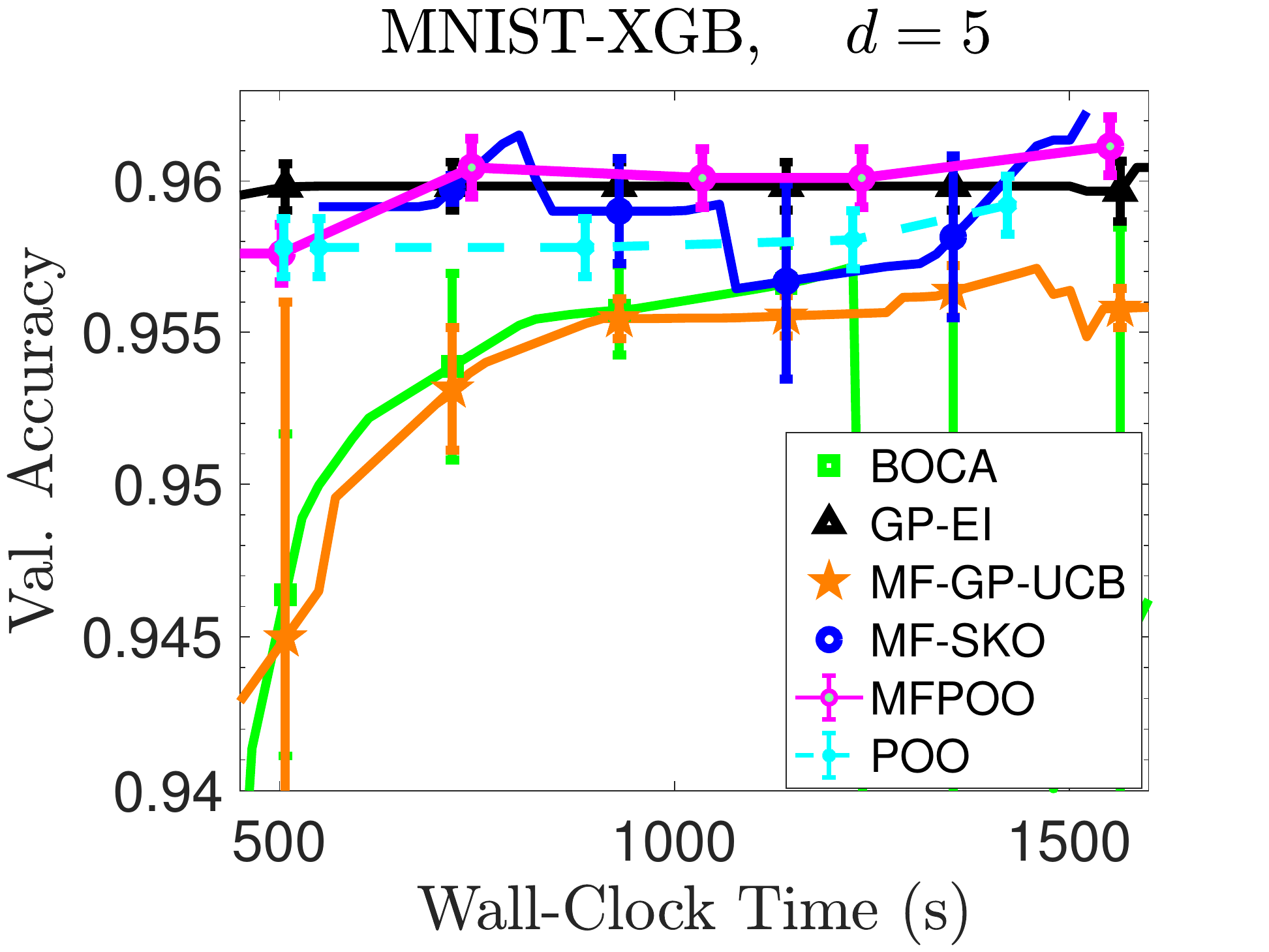}\label{fig:mnist}}\hfill 
	\subfloat[][]{\includegraphics[width = 0.32\linewidth]{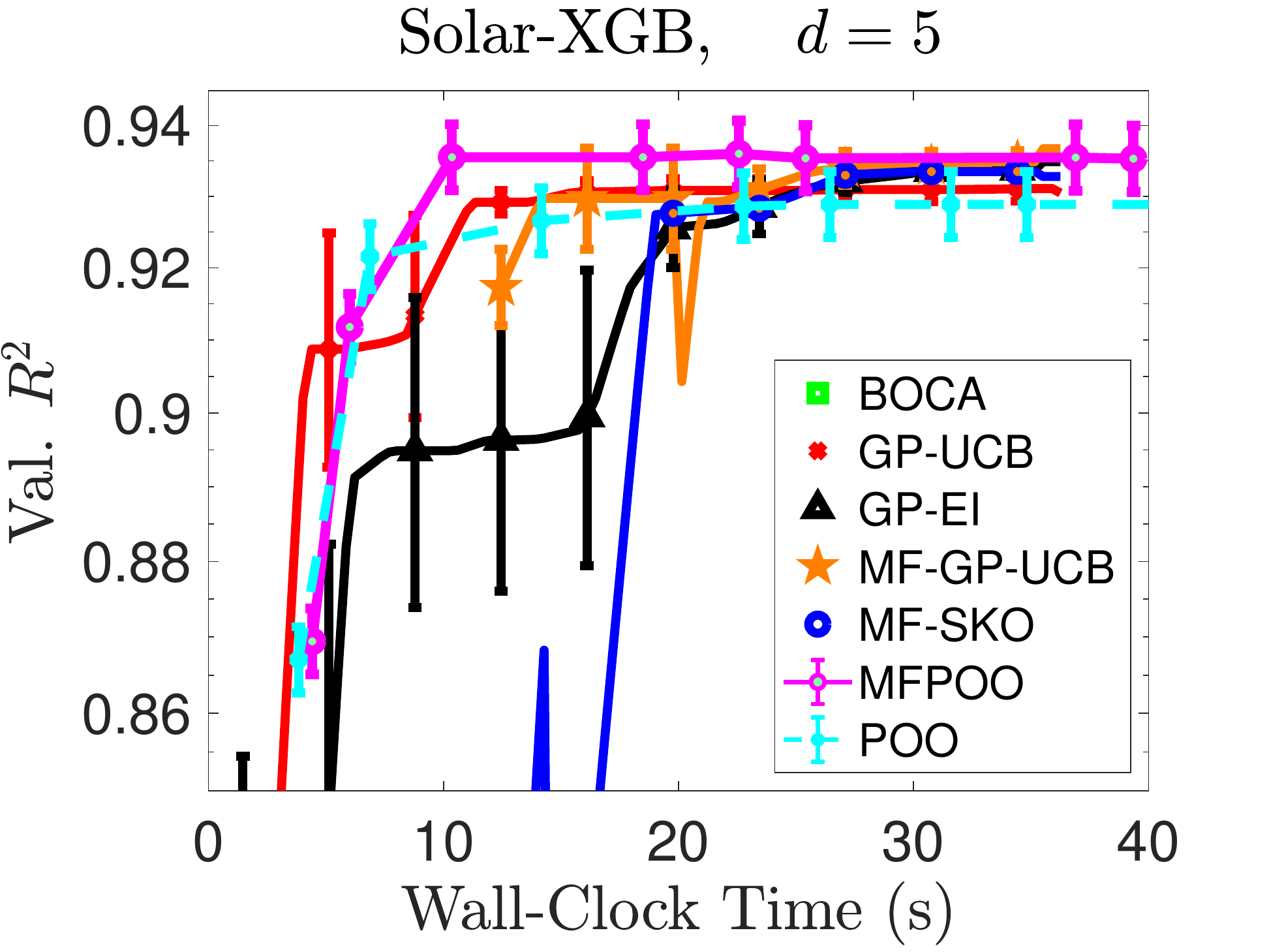}\label{fig:solar}} \\
	\subfloat[][]{\includegraphics[width = 0.32\linewidth]{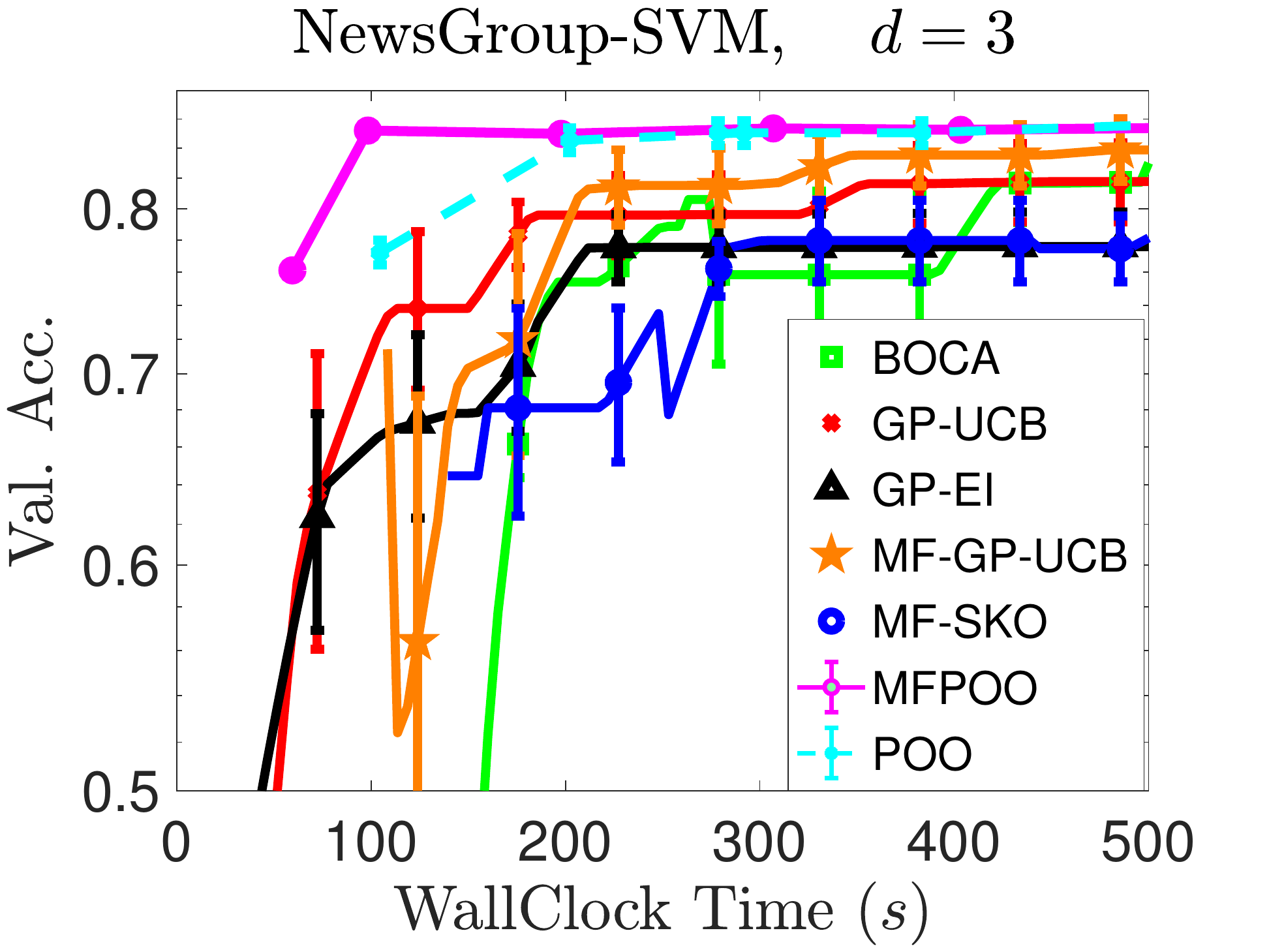}\label{fig:news}}
	\subfloat[][]{\includegraphics[width = 0.32\linewidth]{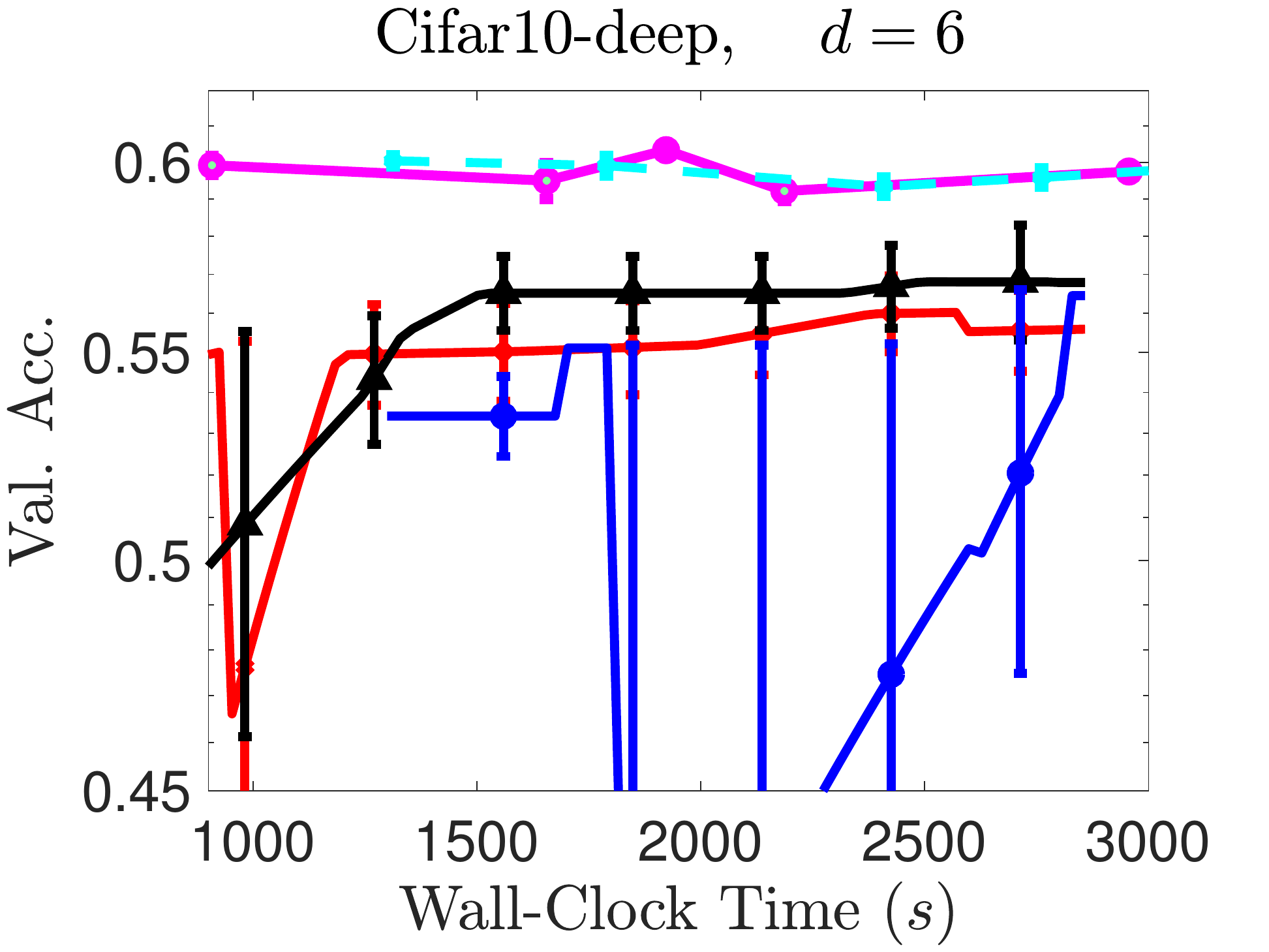}\label{fig:cifar10}} \hfill
	\caption{\small The \textit{common legend} for all the plots is presented in Fig.~\ref{fig:solar}, in the interest of space and clarity. Figures $(a)$ to $(d)$ consists of experiments on multi-fidelity versions of synthetic functions. The experiments are averaged over $10$ runs and the corresponding confidence bars are plotted. Figure $(e)$ shows the  $5$-fold cross-validation accuracy achieved vs. wall-clock time, for tuning XGB on the MNIST data-set. GP-UCB is omitted in this figure due to poor performance.  Figure $(f)$ shows the  $5$-fold cross-validation R-square achieved vs. wall-clock time, for tuning XGB on the Solar-Radiation regression data-set. BOCA is omitted in this figure due to poor performance. Figure $(g)$ shows the performance of the algorithms for tuning SVM on the 20-News Group dataset. Figure $(h)$ shows the comparison of various algorithms for tuning the hyper-parameters of a ConvNet on the Cifar-10 data-set. The code base provided for BOCA and MF-GP-UCB failed to converge for this data-set. All the experiments are averaged over $5$ runs.
	}
	\label{fig:allexp} 
\end{figure*}
In this section we empirically validate the performance of our algorithms as compared to other benchmark algorithms for the multi-fidelity black-box optimization setting on real and synthetic data-sets. We first compare the algorithms on popular synthetic benchmark functions commonly used in the black-box optimization literature. We also  empirically validate the performance of MFPOO against other algorithms for real-world use cases of hyper-parameter tuning. The algorithms under contention are: (i) BOCA~\cite{kandasamy2017multi} which is a multi-fidelity Gaussian Process (GP) based algorithm that can handle continuous fidelity spaces, (ii) MF-GP-UCB~\cite{kandasamy2016multi} which is a GP based multi-fidelity method that can handle finite fidelities, (iii) GP-EI criterion in bayesian optimization~\cite{jones1998efficient}, (iv) MF-SKO, the multi-fidelity sequential kriging optimisation method~\cite{huang2006sequential}, (v) GP-UCB~\cite{srinivas2009gaussian} and (vi) MFPOO (Algorithm~\ref{algo:robust_algo}) and (vii) POO~\cite{grill2015black}. 

In our implementation of MFPOO, we \textit{do not} assume access to a known bias function. In all our experiments it is assumed that the bias function has a parametric form $\zeta(z) = c(1-z)$. The parameter $c$ can be initialized and then updated online owing to the fact that different MFHOO instances spawned by MFPOO query the same node at different fidelities. In all our experiments we set $\rho_{max} = 0.95$. We provide more implementation details in Appendix~\ref{sec:implement}, in the interest of space. All experiments were performed on a 32-core Intel(R) Xeon(R) @ 2.60GHz machine, with a Nvidia 1080 Ti GPU.


{\bf Synthetic Experiments: }
We now provide empirical results on commonly used synthetic benchmark functions. The multi-fidelity setup is introduced into the benchmark functions following the methodology in~\cite{kandasamy2017multi}. The exact details of the functions at different fidelities are provided in Appendix~\ref{sec:appSynthetic}. Note that the bias function is \textit{not} assumed to be known however the cost function is known. We add Gaussian noise in the function evaluations at different variances $\sigma^2$ as specified in Appendix~\ref{sec:appSynthetic}. The performance of the algorithms on 4 different benchmark functions are shown in Fig.~\ref{fig:allexp} (a) - (d). The functions used are Hartmann3, Hartmann6, Branin~\cite{szego1978towards} and CurinExp~\cite{currin1988bayesian}. At the top of each sub-figure, we mention the function name and the dimension of the domain ($d$). We can observe that the tree search based methods (MFPOO and POO) outperform the other benchmarks. Among the two, MFPOO performs better that POO, because it can effectively use multiple fidelities.

{\bf XGB on MNIST: } As our second experiment, we consider the task of tuning XGBOOST~\cite{chen2016xgboost} on the MNIST data-set~\cite{lecun1998gradient}. We consider a a subset consisting of $20000$ images. The black-box function being evaluated is the 5-Fold cross-validation accuracy at the the highest fidelity $z=1$, which refers to using the whole data-set. The fidelity range $\mathcal{Z} = [0,1]$ is mapped to $[500,20000]$, that is using a fidelity $z \in [0,1]$ implies using a randomly sub-sampled data-set consisting of $\floor{z*(15000) + 500}$ samples in order to measure the cross-validation error. The hyper-parameters being tuned and the respective ranges are: \textit{max\_depth}: [2,13], \textit{colsample\_bytree}: [0.2,0.9], \textit{n\_estimators}: [10,400], \textit{gamma}: [0,0.7], \textit{learning\_rate}: [0.05,0.3]. We plot the cross-validation accuracy achieved by different methods as a function of time in Fig.~\ref{fig:mnist}. MFPOO outperforms the other algorithms in terms of validation accuracy achieved. GP-EI is also promising on this data-set. The final cross-validation accuracy achieved by MFPOO and GP-EI are $0.9611$ and $0.9597$ respectively. Note that in this experiment a single experiment at the highest fidelity takes approximately $200$ seconds. The results are averaged over $5$ experiments. $\sigma$ in our algorithm is set to $0.05$. 

{\bf XGB on Solar Data: } We test the algorithms on a regression problem that involved predicting the level of solar radiation given several weather indicators~\cite{solar}. The data-set has $32684$ samples. The fidelities are mapped to the range $[700,32684]$ similar to the MNIST example above. The hyper-parameters and their ranges are also identical to the experiment above. The function value at the highest fidelity is the $5$-Fold cross-validation R-square on the whole data-set. The performances of the algorithms are plotted in Fig.~\ref{fig:solar}. It can be observed that MFPOO outperforms the other algorithms especially in the lower-time horizons. Note that a single experiment at the highest fidelity for this data-set takes $2$ seconds. All experiments were performed on the same machine. The results are averaged over $5$ experiments. More details are in Appendix~\ref{sec:implement}. 

{\bf SVM on 20 News Group:} In Fig.~\ref{fig:news}, we test the algorithms for tuning hyper-parameters of scikit-learn's SVM classifier module on the News Group dataset~\cite{Lang95}. The hyper-parameters to being tuned are: $(i)$ the regularization penalty in the range [1e-5,1e5] (accessed in the log. scale), $(ii)$ the kernel temperature ($\gamma$) also in the range [1e-5,1e5] and $(iii)$ kernel type between \{'rbf','poly'\}. We use a subset of $7000$ samples for training i.e $z = 1$ corresponds to using all $7000$ samples and $z = 0$ corresponds to a randomly chosen subset of size $100$. The black-box function corresponds to the 5-fold cross-validation accuracy at the chosen fidelity. We can observe that MFPOO outperforms the other algorithms especially in lower budget settings. One evaluation at the highest fidelity takes 40 seconds for this experiment. 

{\bf ConvNet on Cifar-10:} In Fig.~\ref{fig:cifar10}, we employ the algorithms for tuning the hyper-parameters of a deep convolutional network for classifying the cifar-10~\cite{krizhevsky2009learning} dataset. As the training set we use a subset of $50k$ samples from the original training data. The black-box function is the accuracy on a fixed validation set (randomly chosen half of the official test set) after 30-epochs. Note that we want to test the relative accuracy obtained by each of the tuning algorithms and therefore in the interest of time we set maximum number of epochs to be $30$, even though higher accuracy can be obtained by training for more epochs. We use the AlexNet~\cite{krizhevsky2012imagenet} architecture. The hyper-parameters being tuned are: $(i)$ number of output channels in first conv. layer in the range [32,128], $(ii)$ kernel size in first layer in [5,14], $(iii)$ number of output channels in second layer in [128,256], $(iv)$ kernel size in second layer in [3,13], $(v)$ learning rate of Adam optimizer in [1e-5,1e-2] accessed in log-scale and $(vi)$ dropout probability in the last layer in the range [0,0.4]. The fidelity is the number of samples used for training where $z = 0$ corresponding to $1000$ randomly chosen training samples while $z=1$ means using $50k$ samples for training. It takes about $600$ seconds for one evaluation at $z = 1$. We can see that both the tree based methods clearly outperform the other algorithms in this experiment. Note that POO does not work for a budget of $1000$ seconds but MFPOO does.

%% file: conc.tex
\section{Conclusion}
We study noisy black-box optimization using \textit{tree-like} hierarchical partitions of the parameter space, when low-cost approximations are available. We propose two algorithms, MFHOO (Algorithm~\ref{alg:mfhoo}) and MFPOO (Algorithm~\ref{algo:robust_algo}) for this problem and provide simple regret guarantees for both our algorithms. Our algorithms are empirically validated against various benchmarks showing superior performance in both simulations and in real world hyper-parameter tuning examples over a wide range of datasets and learning algorithms. We believe that this paper opens up several interesting research problems, for instance developing more adaptive algorithms that query different areas of the domain at different fidelities even at the same height of the tree. We also believe that a more nuanced analysis of the algorithm is possible leading to better simple regret guarantees.

%% file: known_smooth.tex
\section{More on MFPOO (Algorithm~\ref{algo:robust_algo})}
\label{sec:mfpoo}
In this section we will provide more insights about our second algorithm MFPOO (Algorithm~\ref{algo:robust_algo}). We would like to note that in practice if the evaluations of performed by the series of MFHOO instances are stored, then the subsequent MFHOO instances can reuse this information and save on the cost budget. Our implementation can account for this aspect and this provides a significant performance improvement in practice. In Algorithm~\ref{alg:mfhoo} and Algorithm~\ref{algo:robust_algo} it has been assumed that the bias function is known, which may not be true in practice. However, we can assume a simple parameteric form of the bias function and update it online. We will provide more details about this in Appendix~\ref{sec:implement}.

We would also like to highlight the following about Theorem~\ref{thm:mfpoo}, which provides simple regret guarantees for MFPOO. Note that Theorem~\ref{thm:mfpoo} only states that one of the MFHOO instances spawned by Algorithm~\ref{algo:robust_algo} has a simple regret bound stated in the theorem. However, we do not provide any theoretical guarantees regarding the fact that with high probability $i^*$ selected in step 6 of Algorithm~\ref{algo:robust_algo} is the MFHOO instance with the best performance. The analysis of POO~\cite{grill2015black} can provide such a guarantee in a non-multi-fidelity setup by keeping track of the average value of the points evaluated by each of the HOO instances spawned. This analysis cannot be extended as the points evaluated by different MFHOO instances are all at different gradations of the fidelity and thus have varying biases. However, it should be noted that in practice the point $x_{\Lambda,i}$ returned by MFHOO instance $i$ is chosen according to the scheme in Remark~\ref{remark:practice}. Therefore, the function value of the point returned is a good indicator of the overall performance of the MFHOO instance and thus it is expected that the best performing MFHOO instance is selected in step 6. This is also corroborated by the strong empirical performance of MFPOO in our real and synthetic experiments in Section~\ref{sec:sims}.

\section{Regret Guarantees when $(\nu^*,\rho^*)$ are known}
\label{sec:known}
The analysis of Algorithm~\ref{alg:mfhoo} proceeds in these steps:

{\bf (i)} We first prove that $N(\Lambda)$ (the number of steps performed by the algorithm) has to be at least some quantity $n(\Lambda)$ with probability one, when MFHOO is run with a cost budget of $\Lambda$. 

{\bf (ii)} Given that the algorithm performs $n(\Lambda)$ evaluations we prove that the cumulative regret (see the definition in~\cite{bubeck2011x}) incurred by the algorithm till then is bounded by $R_{n(\Lambda)}$. Therefore, owing to step 22 of the algorithm, the simple regret $s(\Lambda)$ is bounded by $R_{n(\Lambda)}/n(\Lambda)$. The main challenge in the analysis is to show that the guarantees similar to that of HOO~\cite{bubeck2011x} can be achieved in the presence of fidelity biases and under the new set of assumptions. 

Lemma~\ref{lem:numsteps} is the main result for Step-\textbf{(i)} of the analysis. 

\begin{lemma}
	\label{lem:numsteps} When Algorithm~\ref{alg:mfhoo} is run with a budget of $\Lambda$, $N(\Lambda) \geq n(\Lambda) + 1$ where,
	\begin{align*}
	n(\Lambda) = \max \{n : \sum_{h = 1}^{n} \lambda(z_h) <  \Lambda\}. 
	\end{align*}
\end{lemma}

\begin{proof}
	Note that the structure of MFHOO is such that at each step a child of a current leaf node is expanded and the child is queried at a fidelity $z_{h+1}$ where $h$ was the height of the leaf node. Also note that for all $h$, $z_{h+1} > z_h$, and therefore $\lambda(z_{h+1}) > \lambda(z_h)$. Suppose, $n$ steps of MFHOO has been performed. The worst case cost in those $n$ steps can therefore be in the case where the algorithm explores without branching, that is at time-step $t \in [n]$ a node at depth $h = t$ is queried. 
	
	In this case the cost incurred is $\sum_{h=1}^{n} \lambda(z_h)$. Therefore, in this dominating corner case the algorithm will query at least $n(\Lambda)$ times. Thus, in all other cases the algorithm is bound to perform at least $n(\Lambda)$ queries. 
\end{proof}

The main result for step \textbf{(ii)} of the analysis if provided as Theorem~\ref{thm:main2}. 

\begin{theorem}
	\label{thm:main2}
	If Algorithm~\ref{alg:mfhoo} is allowed to run for $n$ queries, then the cumulative regret $R_n$ accumulated is bounded as,
	\begin{align*}
	R_n = \mathcal{O} \left(n^{\frac{d(\nu,\rho)+1}{d(\nu,\rho)+2}} (\log n)^{1/(d(\nu,\rho)+2)} \right). 
	\end{align*}
\end{theorem}

The first step in the proof is equivalent to Lemma 14 in~\cite{bubeck2011x}, which we provide below for completeness. 
\begin{lemma}
	\label{lem:bubeck}
	Let $(h,i)$ be a sub-optimal node. Let $0 \leq k \leq h-1$ be the largest height such that $(k,i_k^*)$ is on the path from the root to $(h,i)$. Then for all integers $u \geq 0$, we have,
	\begin{align*}
	&\EE[T_{h,i}(n)] \leq u + 
	\sum_{t=u+1}^{n} \PP \left\{ [U_{s,i_s^*}(t) \leq f^* \text{ for some } s \in \{ k+1,...,t-1\} ] \right. 
	\left.\text{ or } [T_{h,i}(t) > u , U_{h,i}(t) > f^*] \right\} \
	\end{align*}
\end{lemma}

\begin{proof}
	It follows directly from Lemma 14 in~\cite{bubeck2011x}. 
\end{proof}

\begin{lemma}
	\label{lem:opt}
	For all optimal nodes $(h,i)$ and for all integers $n \geq 1$,
	\begin{align*}
	\PP \left\{ U_{h,i}(n) \leq f^* \right\} \leq n^{-3}. 
	\end{align*}
\end{lemma}

\begin{proof}
	We only consider the case where $T_{h,i}(n) \geq 1$, because otherwise the lemma is true trivially. Note that by Assumption~\ref{as:partition}, we have that $f^* - f(x) \leq \nu\rho^h$ for all $x \in (h,i)$. Hence, we have the following,
	\begin{align}
	\sum_{t = 1}^{n} \left( f(X_t) + \nu\rho^h - f^*\right) \mathds{1}_{(H_t,I_t) \in \mathcal{C}(h,i)} \geq 0. 
	\end{align}
	
	We now have the following chain,
	\begin{align*}
	&\PP \left\{ U_{h,i}(n) \leq f^* \right\} \\
	&= \PP \left\{ \hat{\mu}_{h,i}(n) + \sqrt{\frac{2 \sigma^2 \log n}{T_{h,i}(n)}} + 2\nu\rho^h  \leq f^* \right\} \\
	&= \PP \left\{ T_{h,i}(n)\hat{\mu}_{h,i}(n) + T_{h,i}(n)(2\nu\rho^h - f^*) \right. 
	\left.\leq -\sqrt{2\sigma^2 T_{h,i}(n) \log n} \right\} \\
	&\stackrel{(a)}{\leq} \PP \left\{ \sum_{t = 1}^{n} \left( Y_t - f_{Z_t} (X_t) \right) \mathds{1}_{(H_t,I_t) \in \mathcal{C}(h,i)} \right. 
	\left.+ \sum_{t = 1}^{n} \left( f(X_t) + \nu\rho^h - f^*\right) \mathds{1}_{(H_t,I_t) \in \mathcal{C}(h,i)} \right. 
	\left.\leq -\sqrt{2\sigma^2 T_{h,i}(n) \log n} \right\} \\
	& \leq \PP \left\{ \sum_{t = 1}^{n} \left( Y_t - f_{Z_t} (X_t) \right) \mathds{1}_{(H_t,I_t) \in \mathcal{C}(h,i)}  \right. 
	\left. \leq -\sqrt{2\sigma^2 T_{h,i}(n) \log n} \right\}.
	\end{align*}
	
	Here, step (a) follows from the fact that $\zeta(Z_t) \leq \zeta(z_h) = \nu\rho^h$, and therefore $f(X_t) - f_{Z_t}(X_t) \leq \nu\rho^h$. The last term can be bounded by $n^{-3}$ using an union bound and Azuma-Hoeffding for martingale differences, similar to the last part of Lemma 15 in~\cite{bubeck2011x}.  
	
\end{proof}

\begin{lemma}
	\label{lem:sub}
	For all integers $t \leq n$, and for all suboptimal nodes $(h,i)$ such that $\Delta_{h,i} > \nu\rho^h$, and $u \geq 8 \sigma^2 \log n /(\Delta_{h,i} - \nu\rho^h)^2$, we have,
	\begin{align}
	\PP \left\{ U_{h,i}(t) > f^* , T_{h,i}(t) \geq u\right\} \leq tn^{-4}. 
	\end{align}
\end{lemma}

\begin{proof}
	Note that for these values of $u$ we have,
	\begin{align*}
	\sqrt{\frac{2 \sigma^2\log t}{u}} + \nu\rho^h \leq \frac{\Delta_{h,i} + \nu\rho^h}{2}. 
	\end{align*}
	Therefore, we have the following chain,
	\begin{align*}
	& \PP \left\{ U_{h,i}(t) > f^* , T_{h,i}(t) > u \right\} \\
	& = \PP \left\{ \hat{\mu}_{h,i}(t) + \sqrt{\frac{2 \sigma^2\log t}{T_{h,i}(t)}} + 2\nu\rho^h  > f^*_{h,i} + \Delta_{h,i} , T_{h,i}(t) > u \right\} \\
	& \leq \PP \left\{ \hat{\mu}_{h,i}(t) + \nu\rho^h  > f^*_{h,i} + \frac{\Delta_{h,i} - \nu\rho^h}{2} , T_{h,i}(t) > u \right\} \\
	& \leq \PP \left\{ T_{h,i}(t)(\hat{\mu}_{h,i}(t) - (f^*_{h,i} - \nu\rho^h))  > \right. 
	\left.\frac{\Delta_{h,i} - \nu\rho^h}{2}T_{h,i}(t) , T_{h,i}(t) > u \right\} \\
	& \leq \PP \left\{  \sum_{s = 1}^{t} \left( Y_s - f_{Z_s} (X_s) \right) \mathds{1}_{(H_s,I_s) \in \mathcal{C}(h,i)} \right. 
	\left.> \frac{\Delta_{h,i} - \nu\rho^h}{2}T_{h,i}(t) , T_{h,i}(t) > u \right\}.
	\end{align*}
	Now, following the same techniques as in Lemma~\ref{lem:opt} (and also in Lemma-16 in~\cite{bubeck2011x}) it can be shown that the last term is less than $tn^{-4}$. 
\end{proof}

Combining Lemma~\ref{lem:opt} and \ref{lem:sub} we get the following result. 

\begin{lemma}
	\label{lem:submain}
	For all sub-optimal nodes $(h,i)$ with $\Delta_{h,i} > \nu\rho^h$ we have,
	\begin{align*}
	\EE \left[ T_{h,i} (n) \right] \leq \frac{8 \sigma^2 \log n}{(\Delta_{h,i} - \nu\rho^h)^2} + 4.
	\end{align*}
\end{lemma}

\begin{proof}[Proof of Theorem~\ref{thm:main2}]
	Let $H$ be a fixed integer greater than $1$, which is to be chosen later. Let $I_h$ be the nodes at height $h$ that are $2\nu\rho^h$ optimal. Let $\tau_h$ be the set of nodes at height $h$ which are not in $I_h$ but whose parents are in $I_{h-1}$. 
	We will partition the nodes of the infinite tree into three subsets, $\mathcal{T} = \mathcal{T}_1 \cup \mathcal{T}_2 \cup \mathcal{T}_3.$ Let $\mathcal{T}_1$ be all the descendants of $I_{H}$ and the nodes in $I_H$. Let $\mathcal{T}_2 \triangleq \cup_{0 \leq h < H} I_h$. Let $\mathcal{T}_3$ be the all descendants of $\cup_{0 \leq h < H} \tau_h$, including the nodes themselves.  We define the following partitioned cumulative regret quantities,
	\begin{align}
	R_{n,i} = \sum_{t=1}^{n}  (f^* - f(X_t))\mathds{1}_{\{(H_t,I_t) \in \mathcal{T}_i\}}, \text{ for } i = 1,2,3 . 
	\end{align}
	Note that we have, $R_n = \sum_{i = 1}^{3} \EE[R_{n,i}].$ 
	
	{\bf (i)} Let us first bound $ \EE[R_{n,1}]$. Since, all nodes in $\mathcal{T}_1$ are $2\nu\rho^H$ optimal, therefore by Assumption~\ref{as:partition} all points that lie in these cells are $3\nu\rho^H$ optimal. Therefore, we have that $\EE[R_{n,1}] \leq 3\nu\rho^H n$. 
	
	{\bf (ii)} All nodes that belong to $I_h$ are $2\nu\rho^h$ optimal. Therefore, all points belonging to $I_h$ is $3\nu\rho^h$ optimal. Also, by Definition~\ref{def:dimension}, we have that $|I_h| \leq C(\nu,\rho) \rho ^{-d(\nu,\rho)h}$. Thus, we have the following,
	\begin{align*}
	\EE[R_{n,2}] &\leq \sum_{h = 0}^{H-1} 3\nu\rho^h C(\nu,\rho) \rho ^{-d(\nu,\rho)h} \\
	&\leq 3\nu C(\nu,\rho) \sum_{h=0}^{H-1} \rho^{h(1 - d(\nu,\rho))}. 
	\end{align*}
	
	{\bf (iii)} All nodes in $\tau_h$ have their parents in $I_{h-1}$. So, all the points in these nodes are at least $2\nu\rho^{h-1}$ optimal. Therefore, we have the following chain.
	\begin{align*}
	&\EE[R_{n,3}] \leq \sum_{h=1}^{H} 3 \nu \rho^{h-1} \sum_{i: (h,i) \in \tau_h} \EE[T_{h,i}(n)] \\
	&\leq \sum_{h=1}^{H} 3 \nu \rho^{h-1} 2C(\nu,\rho)  \rho ^{-d(\nu,\rho)(h-1)} \left( \frac{8 \sigma^2 \log n}{( \nu\rho^h)^2} + 4\right)
	\end{align*}
	
	Combining the above three steps we arrive at,
	\begin{align*}
	R_n &\leq 3\nu\rho^H + 3\nu C(\nu,\rho) \sum_{h=0}^{H-1} \rho^{h(1 - d(\nu,\rho))} 
	+ \sum_{h=1}^{H} 6 \nu \rho^{h-1} C(\nu,\rho)  \rho ^{-d(\nu,\rho)(h-1)} \left( \frac{8 \sigma^2 \log n}{( \nu\rho^h)^2} + 4\right) \\
	&\leq \alpha_1 n \rho^{H} + \alpha_2 C(\nu,\rho) \rho^{-H(1 + d(\nu,\rho))} \sigma^2 \log n,
	\end{align*}
	where $\alpha_1$ and $\alpha_2$ are universal constants. 
	
	Now, we choose $H$ such that the two terms in the above equations are order-wise equal. This gives us the following regret bound,
	\begin{align}
	R_n \leq \tau C(\nu,\rho)^{1/(d(\nu,\rho)+2)}n^{\frac{d(\nu,\rho)+1}{d(\nu,\rho)+2}} (\sigma^2 \log n)^{1/(d(\nu,\rho)+2)},
	\end{align} 
	where $\tau$ is an universal constant. 
\end{proof}

Now, we can combine the above results to arrive at one of our main results. 

\begin{proof}[Proof of Theorem~\ref{thm:mfhoo}]
	Note that in Step-22 of Algorithm~\ref{alg:mfhoo} one of the points seen so far is randomly chosen. Therefore, if the algorithm has evaluated $n$ points so far and incurred a cumulative regret of $R_n$, then the simple regret so far is given by $R_n/n$. Now, we have the following chain,
	\begin{align*}
	&S(\Lambda) \leq \EE \left[ \EE \left[ \frac{R_{n}}{n} \Bigg \vert N(\Lambda) = n \right] \right] \\
	&\leq \EE \left[ \EE \left[ \tau C(\nu,\rho)^{1/(d(\nu,\rho)+2)}n^{-\frac{1}{d(\nu,\rho)+2}} (\sigma^2 \log n)^{1/(d(\nu,\rho)+2)} \Bigg \vert \right. 
	\left. N(\Lambda) = n \right] \right] .
	\end{align*}
	Note, that the expression inside the conditional expectation is decreasing with the value of $n$. Also, according to Lemma~\ref{lem:numsteps} $N(\Lambda) \geq n(\Lambda)$ almost surely. Therefore, we have
	\begin{align*}
	S(\Lambda) \leq \tau C(\nu,\rho)^{\frac{1}{d(\nu,\rho)+2}} n(\Lambda)^{-\frac{1}{d(\nu,\rho)+2}} (\sigma^2 \log n(\Lambda))^{1/(d(\nu,\rho)+2)}.
	\end{align*}
\end{proof}

%% file: relating_to_optimal.tex
\section{Recovering optimal scaling with unknown smoothness}
\label{sec:optimal}
In this section we will prove Theorem~\ref{thm:mfpoo}. The proof of this theorem is very similar to the analysis in~\cite{grill2015black}. We will first use a function lemma from~\cite{grill2015black} that will be key in proving Theorem~\ref{thm:mfpoo}. 

\begin{lemma}
	\label{lem:grill}
	Consider the parameters $\nu > \nu^*$ and $\rho > \rho^*$. Let $h_{min} \triangleq \log (\nu/\nu^*) \log(1/\rho) $. Then we have the following, 
	\begin{align*}
	&\mathcal{N}_h(2\nu\rho^h) \leq \max \left( C(\nu^*,\rho^*) 2 ^{(\log \rho ^* + \log \nu ^* - \log \nu)/\log \rho} , 2^{h_{min}} \right) 
	\times \rho^{-h \left[ d(\nu^*,\rho^*) + \log 2 ( 1/\log(1/\rho) - 1/\log(1/\rho ^*) ) \right]}
	\end{align*}
\end{lemma}

\begin{proof}
	It follows directly from the analysis of Theorem~1 in appendix B.1 of~\cite{grill2015black}. 
\end{proof}

Lemma~\ref{lem:grill} implies the following,
\begin{align*}
C(\nu,\rho) &\leq \max \left( C(\nu^*,\rho^*) 2 ^{(\log \rho ^* + \log (\nu ^*/\nu) )/\log \rho} , 2^{h_{min}} \right) \\
d(\nu,\rho) &\leq d(\nu^*,\rho^*) + \log 2 ( 1/\log(1/\rho) - 1/\log(1/\rho ^*) ) \numberthis \label{eq:relations}
\end{align*}

\begin{proof}[Proof of Theorem~\ref{thm:mfpoo}]
Let $S(\Lambda)_{(\nu,\rho)}$ the simple regret of an MFHOO instance run with parameters $(\nu,\rho)$ satisfying the condition in Theorem~\ref{thm:mfpoo}, given a budget of $\Lambda$. Then from the proof of Theorem~\ref{thm:mfpoo} we have the following chain,
\begin{align*}
& \log S(\Lambda)_{(\nu,\rho)} \leq \log \tau + \frac{\log C(\nu,\rho)}{2 + d(\nu,\rho)} - \frac{\log (n(\Lambda) /(\sigma^2\log n(\Lambda)))}{2 + d(\nu,\rho)} \\
& \leq \log \tau + \frac{\log C(\nu,\rho)}{2 + d(\nu,\rho)} 
- \frac{\log (n(\Lambda) /(\sigma^2\log n(\Lambda)))}{2 + d(\nu^*,\rho^*)} \left(1 - \frac{ d(\nu,\rho)- d(\nu^*,\rho^*)}{2 + d(\nu^*,\rho^*)} \right).
\end{align*}  
The last inequality follows from Eq.~\eqref{eq:relations}. Recall that MFPOO spawns $N$ (defined in Algorithm~\ref{algo:robust_algo}) MFHOO instances each with budget $\Lambda/N$. By Equation.~\eqref{eq:relations}, we have that out of the $N$ parameters $\rho_1,...,\rho_N$, there is at least one say $\bar{\rho} \geq \rho^*$ such that,
\begin{align*}
d(\nu_{max},\bar{\rho}) - d(\nu^*,\rho^*) \leq \frac{D_{max}}{N}. 
\end{align*}
Thus we have that,
\begin{align*}
&\log S(\Lambda/N)_{(\nu_{max},\bar{\rho})} \leq \log \tau + \frac{\log C(\nu_{max},\bar{\rho})}{2 + d(\nu,\bar{\rho})} 
+ \log \left( \frac{\log n(\Lambda/N)}{n(\Lambda/N)}\right) \left(\frac{1}{d(\nu^*,\rho^*)} - \frac{D_{max}/N}{(2 + d(\nu^*,\rho^*))^2} \right). \numberthis \label{eq:b1}
\end{align*}
Using Equation~\ref{eq:relations} and following the steps in B.3 on page 12 in~\cite{grill2015black} it can be shown that,
\begin{align}
\frac{\log C(\nu_{max},\bar{\rho})}{2 + d(\nu,\bar{\rho})} \leq \beta + \frac{D_{max}}{2 + d(\nu^*,\rho^*)} \log (\nu_{max}/\nu^*) \label{eq:b3}
\end{align}
Finally using the fact that $N = 0.5D_{max} \log (\Lambda/\log \Lambda)$ and that $n(\Lambda) \leq \Lambda$ we have the following:
\begin{align}
-\log \left( \frac{\log n(\Lambda/N)}{n(\Lambda/N)}\right) \frac{D_{max}/N}{(2 + d(\nu^*,\rho^*))^2} \leq 2. \label{eq:b4}
\end{align}
Putting together Equation~\eqref{eq:b4},\eqref{eq:b3} and \eqref{eq:b1} we have the following,
\begin{align*}
&S(\Lambda/N)_{(\nu_{max},\bar{\rho})} \leq \tau \exp(\beta + 2) \left( D_{max} (\nu_{max}/\nu^*)^{D_{max}} \right. 
 \left.\log n(\Lambda/\log(\Lambda)) / (n(\Lambda/\log(\Lambda))) \right)^{1/(2 + d(\nu^*,\rho^*))}. 
\end{align*}
This proves that at least one of the MFHOO instances spawned has the regret specified in Theorem~\ref{thm:mfpoo}. 
\end{proof}

%% file: final_guarantees.tex
\section{More on Experiments}
\label{sec:moreexp}

\subsection{Implementation Details}
\label{sec:implement}

In this section, we provide the following implementation details about our algorithm:

{\bf Updating the Bias Function: } As mentioned before, we assume that the bias function if of the form $\zeta(z) = c(1 -z)$. The parameter $c$ is estimated online as follows: (i) We start by choosing a random point $x \in \mathcal{X}$, which is queried at $z_1 = 0.8$ and $z_2 = 0.2$, giving observations $Y_1$ and $Y_2$. We initialize $c = 2|Y_1 - Y_2|/|z_1 - z_2|$. We also set $\nu_{max} = 2*c$. Note that this uses up a small portion of the budget ($\lambda(0.2)+\lambda(0.8)$). The structure of MFPOO is such that while running the parallel MFPOO instances the sames cells (representative point in the cell) is queried again at different fidelities say $z_1$ and $z_2$, yielding function values $Y_1$ and $Y_2$. If at any point $|Y_1 - Y_2|/|z_1 - z_2| > c$, we update $c \leftarrow 2c$. 

{\bf Saving on Parallel MFHOO's: } In practice we can save significant portions of the cost budget by making use of the fact that two MFHOO instances can query the same cell at fidelities $z_1$ and $z_2$ which are very close to each other. We set of tolerance $\tau = 0.01$. If an MFHOO (spawned by MFPOO) queries a cell at a fidelity $z_2$, and that cell had already been queried before at $z_1$, then we reuse the previous evaluation if $|z_1 - z_2| < \tau$. This provides significant gains in practice.

{\bf Hierarchical Partitions: } The hierarchical partitioning scheme followed is similar to that of the DIRECT algorithm~\cite{finkel2003direct}. Each time when a cell needs to be broken into children cells, the coordinate direction in which the cell width is maximum is selected, and the children are divided into halves in the direction of that coordinate. 

{\bf Real-Data Implementations: } Our tree-search implementations are python objects that can take in as input a wrapper class which converts a classification/regression problem into a black-box function object with multiple fidelities. We implement our regressors and classifiers (scikit-learn XGB) within the black-box function objects. We use a 16 Core machine, where XGBoost can be run on parallel threads. We set $nthreads = 5$ and the $5$-Fold cross-validation is also performed in parallel.  

\subsection{Description of Synthetic Functions}
\label{sec:appSynthetic}

We use multi-fidelity versions of commonly used benchmark functions in the black-box optimization literature. These multi-fidelity versions have been previously used in~\cite{kandasamy2016mfgpucb,sen2018multi}.

	\textbf{Currin exponential function \cite{currin1988bayesian}:}
	This is a two dimensional function with domain $[0,1]^2$. The cost function is $\lambda(z) = 0.1 + z^2$ and the noise variance is $\sigma^2 = 0.5$. The multi-fidelity object as a function of $(x,z)$ is, 
	\begin{align*}
	f_z(x) &= \left(1-0.1(1-z)\exp\left(\frac{-1}{2x_2}\right)\right) \times 
	\left(\frac{2300x_1^3 + 1900x_1^2 + 2092x_1 + 60}{100x_1^3 + 
		500x_1^2 + 4x_1 + 20}\right).
	\end{align*}

	\textbf{Hartmann functions \cite{szego1978towards}:}
	We use two Hartmann functions in $3$ and $6$ dimensions. The functional form of the multi-fidelity object is 	$f_z(x) = \sum_{i=1}^4 (\alpha_i - \alpha'(z)) \exp\big( -\sum_{j=1}^3 A_{ij}(x_j-P_{ij})^2\big)$ where $\alpha = [1.0, 1.2, 3.0, 3.2]$ and $\alpha'(z) = 0.1(1-z)$.
	In the case of $3$ dimensions, the cost function is $\lambda(z) = 0.05 + (1 - 0.05)z^3$,
	$\sigma^2=0.01$ and,
	\[
	A = 
	\begin{bmatrix}
	3 & 10 & 30 \\
	0.1 & 10 & 35 \\
	3 & 10 & 30 \\
	0.1 & 10 & 35 
	\end{bmatrix},
	\quad
	P = 10^{-4} \times
	\begin{bmatrix}
	3689 & 1170 & 2673 \\
	4699 & 4387 & 7470 \\
	1091 & 8732 & 5547 \\
	381 & 5743 & 8828
	\end{bmatrix}.
	\]
	Moving to the $6$ dimensional case, the cost function is $\lambda(z) = 0.05 + (1 -
	0.05)z^3$, $\sigma^2=0.05$ and,
	\[
	A = 
	\begin{bmatrix}
	10 & 3 & 17 & 3.5 & 1.7 & 8 \\
	0.05 & 10 & 17 & 0.1 & 8 & 14 \\
	3 & 3.5 & 1.7 & 10 & 17 & 8 \\
	17 & 8 & 0.05 & 10 & 0.1 & 14 
	\end{bmatrix},
	\;\;
	P = 10^{-4} \times
	\begin{bmatrix}
	1312 & 1696 & 5569 &  124 & 8283 & 5886 \\
	2329 & 4135 & 8307 & 3736 & 1004 & 9991 \\
	2348 & 1451 & 3522 & 2883 & 3047 & 6650 \\
	4047 & 8828 & 8732 & 5743 & 1091 &  381 \\
	\end{bmatrix}.
	\].
	
	When $z = 1$, these functions reduce to the commonly used Hartmann benchmark functions.

	\textbf{Branin function \cite{szego1978towards}:}
	For this function the domain is $\Xcal=[[-5, 10], [0, 15]]^2$. The multi-fidelity object is given by,
	\[
	f_z(x) = a(x_2 - b(z)x_1^2 + c(z)x_1 - r)^2 + s(1-t(z))\cos(x_1) + s,
	\]
	where $a = 1$, $b(z)=5.1/(4\pi^2) - 0.01(1-z)$
	$c(z) = 5/\pi - 0.1(1-z)$, $r=6$, $s=10$ and $t(z)=1/(8\pi) + 0.05(1-z)$.
	At $z = 1$, this becomes the standard Branin function.
	The cost function is $\lambda(z) = 0.05 + z^{3}$ and
	$\sigma^2=0.05$ is the noise variance.